\documentclass[twoside,11pt]{article}
\usepackage{jmlr2e}

\usepackage[numbers]{natbib}
\PassOptionsToPackage{numbers, compress}{natbib}

\newcommand{\newadd}[1]{{\color{black}#1}}
\newcommand{\revadd}[1]{{\color{black}#1}}

\usepackage[utf8]{inputenc} 
\usepackage[T1]{fontenc}
\usepackage{url}
\usepackage{hyperref}       
\usepackage{booktabs}       
\usepackage{amsfonts}       
\usepackage{nicefrac}       
\usepackage{microtype}      
\usepackage{ifthen}
\usepackage[linesnumbered,ruled,algo2e]{algorithm2e}
\usepackage{amsthm, amssymb, amsmath}
\usepackage{graphicx}
\usepackage{cleveref}
\usepackage[colorinlistoftodos,prependcaption]{todonotes}

\usepackage{epstopdf}
\usepackage{algorithm}
 \usepackage{algorithmic}
\usepackage{wrapfig}
\usepackage{subcaption}
\usepackage{bbm}

\newcommand{\numArms}{K}

\newcommand{\arm}{k}
\newcommand{\meanReward}{\mu}
\newcommand{\Index}{I}
\newcommand{\pulls}{n}
\newcommand{\slot}{t}
\newcommand{\gap}{\Delta}
\newcommand{\estimateMean}{\hat{\phi}}
\newcommand{\expectedPseudoReward}{\phi}

\newcommand{\reward}{r}
\newcommand{\estimateReward}{s}
\newcommand{\optimistGap}{\tilde{\Delta}}
\newcommand{\indicator}{\mathbbm{1}}

\def \OO {\mathrm{O}}

\newtheorem{lem}{Lemma}
\newtheorem{thm}{Theorem}
\newtheorem{defn}{Definition}
\newtheorem{coro}{Corollary}

\newtheorem{rem}{Remark}
\newtheorem{fact}{Fact}

\newcommand{\E}[1]{\mathbb{E}\left[{#1}\right]}

\DeclareMathOperator*{\argmax}{arg\,max}

\newcommand{\constant}{\zeta}
\newcommand{\B}{b}
\newcommand{\Confidence}{B}

\graphicspath{{Figures/}}

\crefname{equation}{}{}
\Crefname{equation}{}{}
\crefname{thm}{theorem}{theorems}
\Crefname{thm}{Theorem}{Theorems}
\crefname{clm}{claim}{claims}
\Crefname{clm}{Claim}{Claims}
\Crefname{coro}{Corollary}{Corollaries}
\Crefname{lem}{Lemma}{Lemmas}
\Crefname{sec}{Section}{Sections}
\crefname{app}{appendix}{appendices}
\Crefname{app}{Appendix}{Appendices}
\crefname{prop}{proposition}{propositions}
\Crefname{prop}{Proposition}{Propositions}
\Crefname{propty}{Property}{Properties}
\crefname{figure}{figure}{figures}
\Crefname{figure}{Figure}{Figures}
\crefname{fig}{figure}{figures}
\Crefname{fig}{Figure}{Figures}
\crefname{defn}{definition}{definitions}
\Crefname{defn}{Definition}{Definitions}
\crefname{fact}{fact}{facts}
\Crefname{fact}{Fact}{Facts}
\crefname{appendix}{appendix}{appendices}
\Crefname{appendix}{Appendix}{Appendices}
\crefname{algo}{algorithm}{algorithms}
\Crefname{algo}{Algorithm}{Algorithms}
\crefname{algorithm}{algorithm}{algorithms}
\Crefname{algorithm}{Algorithm}{Algorithms}
\crefname{conj}{conjecture}{conjectures}
\Crefname{conj}{Conjecture}{Conjectures}
\crefname{obs}{observation}{observations}
\Crefname{obs}{Observation}{Observations}

\begin{document}

\title{Best-Arm Identification in Correlated Multi-Armed Bandits}

\author{\name Samarth Gupta \email samarthg@andrew.cmu.edu \\
 \addr Carnegie Mellon University\\
 Pittsburgh, PA 15213 
 \AND
 \name Gauri Joshi \email gaurij@andrew.cmu.edu \\
 \addr Carnegie Mellon University\\
 Pittsburgh, PA 15213 
 \AND
 \name Osman Ya\u{g}an \email oyagan@andrew.cmu.edu\\
 \addr Carnegie Mellon University\\
 Pittsburgh, PA 15213}

\editor{No editors}
\maketitle
          
\begin{abstract}
In this paper we consider the problem of best-arm identification in multi-armed bandits in the fixed confidence setting, where the goal is to identify, with probability $1-\delta$ for some $\delta>0$, the arm with the highest mean reward in minimum possible samples from the set of arms $\mathcal{K}$. Most existing best-arm identification algorithms and analyses operate under the assumption that the rewards corresponding to different arms are independent of each other. We propose a novel correlated bandit framework that captures domain knowledge about correlation between arms in the form of upper bounds on expected conditional reward of an arm, given a reward realization from another arm. Our proposed algorithm C-LUCB, which generalizes the LUCB algorithm utilizes this partial knowledge of correlations to sharply reduce the sample complexity of best-arm identification. More interestingly, we show that the total samples obtained by C-LUCB are of the form $\OO\left(\sum_{k \in \mathcal{C}} \log\left(\frac{1}{\delta}\right)\right)$ as opposed to the typical $\OO\left(\sum_{k \in \mathcal{K}} \log\left(\frac{1}{\delta}\right)\right)$ samples required in the independent reward setting. The improvement comes, as the $\OO(\log(1/\delta))$ term is summed only for the set of \emph{competitive} arms $\mathcal{C}$, which is a subset of the original set of arms $\mathcal{K}$. The size of the set $\mathcal{C}$, depending on the problem setting, can be as small as $2$, and hence using C-LUCB in the correlated bandits setting can lead to significant performance improvements. Our theoretical findings are supported by experiments on the Movielens and Goodreads recommendation datasets.
\end{abstract}

\section{Introduction}
The \emph{multi-armed bandit} (MAB) problem falls under the class of sequential decision making problems. In the classical multi-armed bandit setting, the player is asked to sample one of the $K$ arms at every round $t=1, 2, \ldots$. Upon sampling arm  $k_{t}$ at round $t$, the player receives a {\em random} reward $R_{t}$ drawn from the reward distribution of arm $k_{t}$. These reward distributions are assumed to be {\em unknown} to the player, and the most commonly studied objective  is to maximize the {\em long-term} cumulative reward; e.g., see the early work by Lai and Robbins \cite{lai1985asymptotically}. Since then, the reward maximization problem has received attention in both classical settings \cite{auer2002finite, agrawal2013further} and in variants of the classical multi-armed bandits such as linear \cite{abbasi2011improved}, contextual \cite{li2010contextual}, structured bandits \cite{combes2017minimal} etc.

\begin{figure}[t]
    \centering
    \includegraphics[width = 0.7\textwidth]{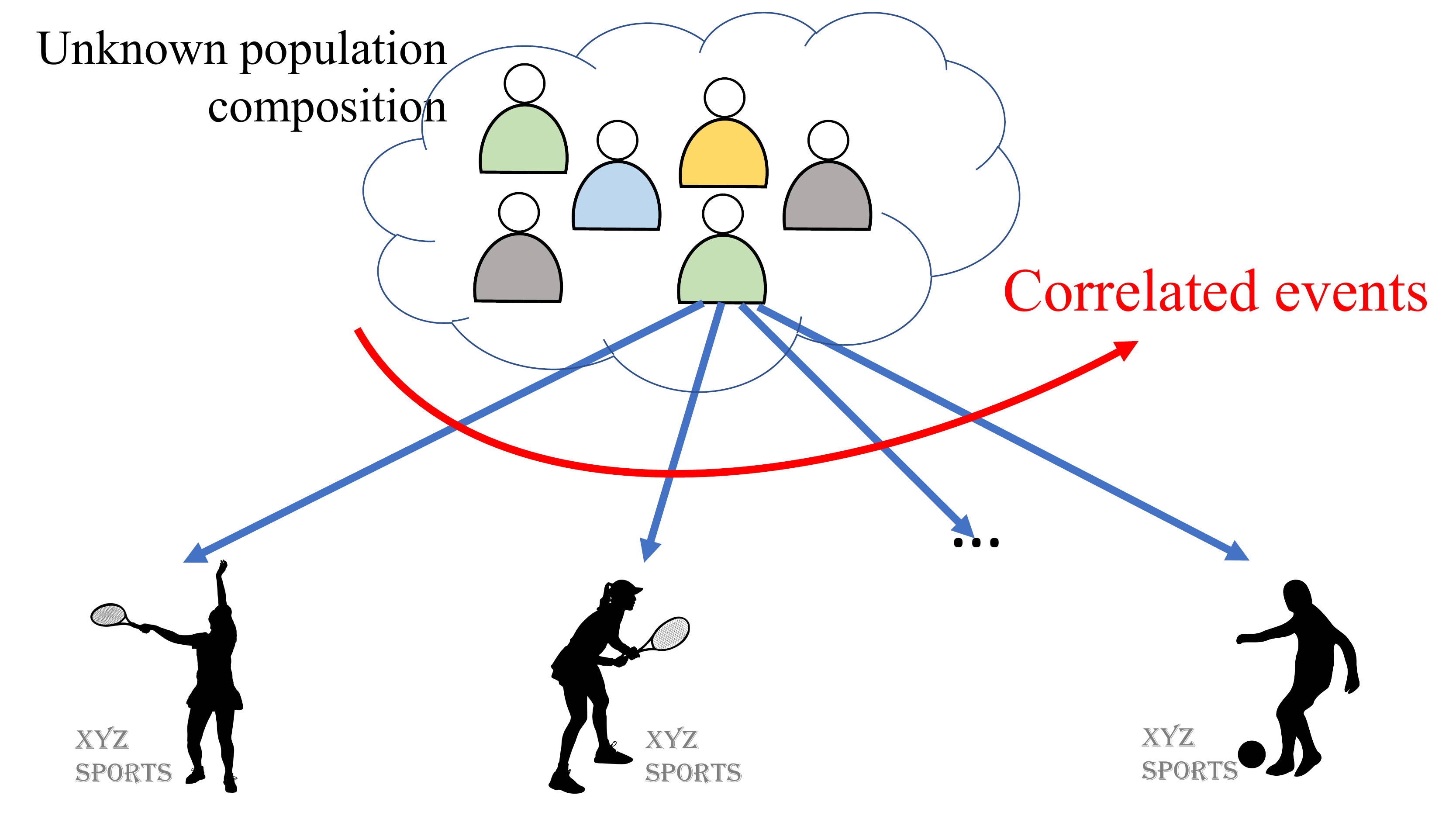}
    \caption{The ratings of a user corresponding to different versions of the same ad are likely to be correlated. For example, if a person likes first version, there is a good chance that they will also like the 2nd one as it also related to tennis. However, the population composition is unknown, i.e., the fraction of people liking the first/second or the last version is unknown.} 
    \label{fig:clooneyEx}
    \vspace{-0.2cm}
\end{figure}

\vspace{0.1cm} 
\noindent
\textbf{Best-arm Identification in Bandits with Independent Arms.} Instead of maximizing the cumulative reward, an alternative objective in the Multi-Armed Bandit setting is to identify the \emph{best arm} (i.e., the arm with the largest mean reward) from as few samples as possible. While reward maximization has been studied extensively, the best-arm identification problem is seldom explored in settings outside of the {\em classical MAB framework}, i.e., the setting where rewards corresponding to different arms are independent of each other. The best-arm identification problem can be formulated in two different ways, namely fixed confidence \cite{jamieson2014best} and fixed budget \cite{bubeck2009pure}.
In the fixed confidence setting, the player is provided with a {\em confidence parameter} $\delta$ and their goal is to achieve the fastest (i.e., with the least number of samples) possible identification of the best arm with a probability of at least $1 - \delta$. In the fixed budget setting, the number of samples that the player can receive is fixed, and the goal is to identify the best arm with the highest possible confidence. 
In this paper, we focus on the fixed confidence setting.

The best arm identification problem has been explored in the classical MAB framework
\cite{jamieson2014lil, kaufmann2013information, tanczos2017kl, simchowitz2017simulator,kalyanakrishnan2012pac, bechhofer1958sequential, even2002pac} and three distinct approaches have shown promise, namely, the racing/successive elimination, law of iterated logarithm upper confidence bound (lil'UCB) and lower and upper confidence bound (LUCB) based approaches. These algorithms maintain upper and lower confidence bound indices for each arm and usually stop once the lower confidence index of one arm becomes larger than upper confidence bound of all other arms (discussed in more detail in \Cref{sec:priorWork}). These three approaches differ in their approach of sampling arms. The successive elimination approach samples arms in a round robin manner, lil'UCB samples the arm with the largest upper confidence bound index at round $t$ and LUCB samples two distinct arms at each round, first it samples the arm with the largest empirical mean and then amongst the rest it samples an arm with the largest upper confidence bound index.

These best-arm identification algorithms have found their use in a wide variety of application settings, such as clinical trials \cite{villar2015multi} , ad-selection campaigns \cite{white2012bandit} , crowd-sourced ranking \cite{tanczos2017kl} and hyperparameter optimization \cite{li2017hyperband} by treating different different drugs/treatments, advertisements, items to be ranked and hyperparameters as the arms in the multi-armed bandit problem. 

\begin{figure}[t]
    \centering
    \includegraphics[width = 0.6\textwidth]{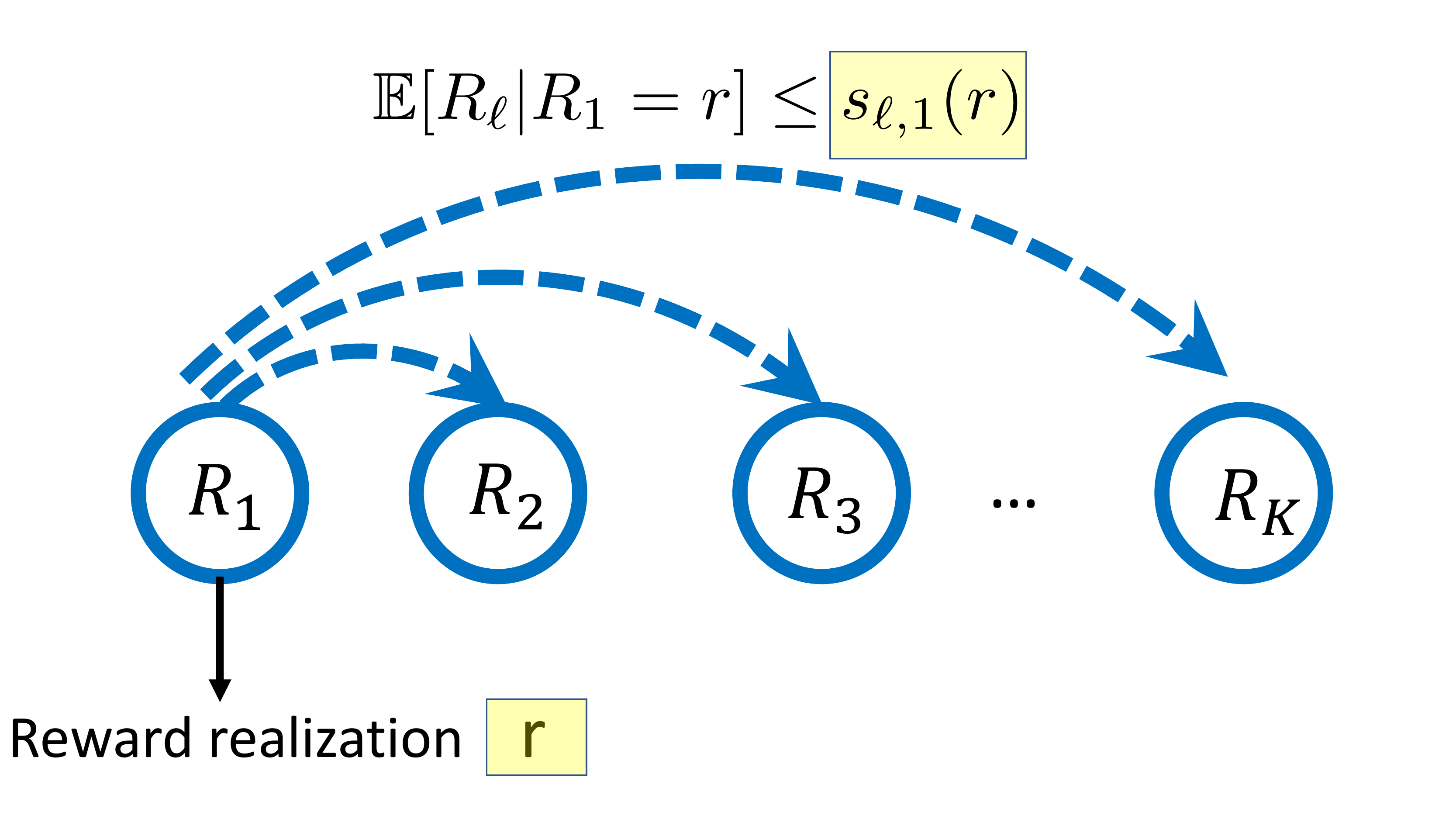}
    \caption{ Upon observing a reward $r$ from an arm $k$, pseudo-rewards $s_{\ell,k}(r),$ give us an upper bound on the conditional expectation of the reward from arm $\ell$ given that we observed reward $r$ from arm $k$. These pseudo-rewards models the correlation in rewards corresponding to different arms. 
    }
    \label{fig:pseudoModel}
\end{figure}

\vspace{0.2cm} \noindent
\textbf{Best-arm Identification when Rewards are Correlated across arms.} The aforementioned best-arm identification algorithms all operate under the assumption that the rewards  from different arms are independent of each other; e.g., at a given round $t$, the reward obtained from arm $k$ does not provide any information about the reward that one might have received if they sampled another arm $\ell$. However, this may not be the case in many applications of MABs.
For instance,  the response of a user for different advertisements in an ad-campaign is likely to be correlated as the ad designs may be related or starkly different with each other (see \Cref{fig:clooneyEx}). \newadd{One  way to learn these correlations would be to pull multiple arms at each round $t$. Since this is not allowed in the standard MAB setup, we assume that {\em partial} information about such correlations is available a priori.} In practice, the presence of such correlations may be known beforehand either through domain expertise or through controlled studies where each user is presented with multiple arms.  For example, before starting ad campaign, partial information may be known about the {\em expected} reward we would receive from a user by showing that ad version $\ell$, given their response to version $k$. 
A similar argument can be made in the application domain of clinical trials, namely in identifying the best drug for an unknown disease. There, the effect of different drugs on an individual may be correlated if the drugs share similar or contrasting components among them. In this context, the correlations would be expected to be known by the domain expertise of the physicians involved. The current best-arm identification algorithms cannot leverage these correlations to reduce the number of samples required in identifying the best arm. This papers  aims to fill this gap in the literature through a new MAB model introduced next.

\vspace{0.2cm} 
\noindent
\textbf{A Novel Correlated MAB model.} Motivated by this, we consider a multi-armed bandit framework where rewards corresponding to different arms are correlated. We model the \emph{partial} knowledge of correlations through \emph{pseudo-rewards} that represent  upper bounds on the conditional mean rewards. The pseudo-rewards provide us an upper bound on the expected reward from arm $\ell$, given that the response from arm $k$ was $r$ (See \Cref{fig:pseudoModel}), i.e.,
\begin{equation}
  \E{R_{\ell} | R_k = r} \leq s_{\ell,k}(r). 
\end{equation}
A key advantage of this model is that pseudo-rewards are just \emph{upper bounds} on the conditional expected reward and they can be arbitrarily loose. In the case where all bounds are trivial, our framework reduces to that of the classical Multi-armed bandit setting. This model was first proposed by us in \cite{gupta2019multi}, where we studied the problem of reward maximization. Two seemingly related models are the structured \cite{gupta2018unified, huang2017structured} and contextual \cite{li2010contextual} multi-armed bandit models.

\noindent
\textbf{Comparison with Contextual and Structured bandits: } In contextual bandits,  the context features of the user (i.e., the user to whom ad is recommended) are assumed to be known, and the goal is to learn a mapping from the context features to the expected rewards so that  each user can be given a {\em personalized} recommendation. In contrast, our model focuses on a setting where context features of the users are {\em not} known and the goal is to find a {\em single} recommendation for the {\em entire} demographic. 

Our work falls under the class of structured bandits, which in its full generality, poses restrictions on the joint probability distribution of rewards. To the best of our knowledge, existing work on best-arm identification   in structured bandits focus on settings where mean rewards of the arms are related to one another through a hidden parameter $\theta$. In particular, the mean reward of arm $k$ is $\mu_k(\theta)$, where $\theta$ is a hidden parameter common to all $K$ arms. It assumes that the mean reward mappings $\mu_k(\theta)$ are known beforehand, but the hidden parameter is unknown. While the mean rewards are related to one another in these works, the rewards are not necessarily correlated. A more detailed comparison is presented in \Cref{sec:priorWork}. In this work, we explicitly model the correlation through knowledge of pseudo-rewards.

\vspace{0.2cm} 
\noindent
\textbf{Proposed C-LUCB Algorithm and its Sample Complexity.} After establishing a correlated bandit model, we then focus on designing best-arm identification algorithms, that are able to make use of this correlation information to identify the best-arm in fewer samples than the classical best-arm identification algorithms. In particular, we propose an approach that makes use of the pseudo-reward information and extends the LUCB approach to the correlated bandit setting. Our sample complexity analysis shows that the proposed C-LUCB approach is able to explore certain arms without explicitly sampling them. Due to this, we see that these arms, termed as non-competitive contribute only an $\OO(1)$ term in the sample complexity as to the typical $\OO\left(\log\frac{1}{\delta}\right)$ contribution by each arm. As a result of this, we are able to provide better sample complexity results than LUCB in the correlated bandit setting. In particular, the LUCB algorithm stops with probability $1 - \delta$ after obtaining at most $\sum_{k \in \mathcal{K}} \frac{2\constant}{\Delta_k^2} \left(\log\left(\frac{K\log\left(\frac{1}{\Delta_k^2}\right)}{\delta}\right)\right)$ samples, where $\Delta_k = \mu_{k^*} - \mu_k$, i.e., the difference in mean reward of optimal arm $k^*$ and mean reward of arm $k$ and $\Delta_{k^*} = \min_{k \neq k^*} \Delta_k$, i.e., the gap between best and second best arm and $\constant > 0$ is a constant. The C-LUCB stops after at most $\sum_{k \in \mathcal{C}} \frac{2\constant}{\Delta_k^2} \left(\log\left(\frac{2K\log\left(\frac{1}{\Delta_k^2}\right)}{\delta}\right)\right) + \OO(1)$ samples with probability $1 - \delta$. Here, $\mathcal{C} \subseteq \mathcal{K}$ with $2 \leq |\mathcal{C}| \leq K $ depending on the problem instance. As the size of the set $\mathcal{C}$ can be smaller than $\mathcal{K},$ we improve upon the sample complexity results of standard approaches of best-arm identification. This theoretical advantage gets reflected in our experiments on two real-world recommendation datasets, namely, Movielens and Goodreads.  For instance, \Cref{fig:introFig} illustrates the performance of our proposed algorithms in a correlated bandit framework, where the goal is to identify the best movie genre from the set of 18 movie genres in the Movielens dataset. As our proposed approach utilizes the correlations in the problem, they draw fewer samples than the Racing, lil'UCB and the LUCB based approaches. 

\begin{figure}[t]
    \centering
    \includegraphics[width = 0.9\textwidth]{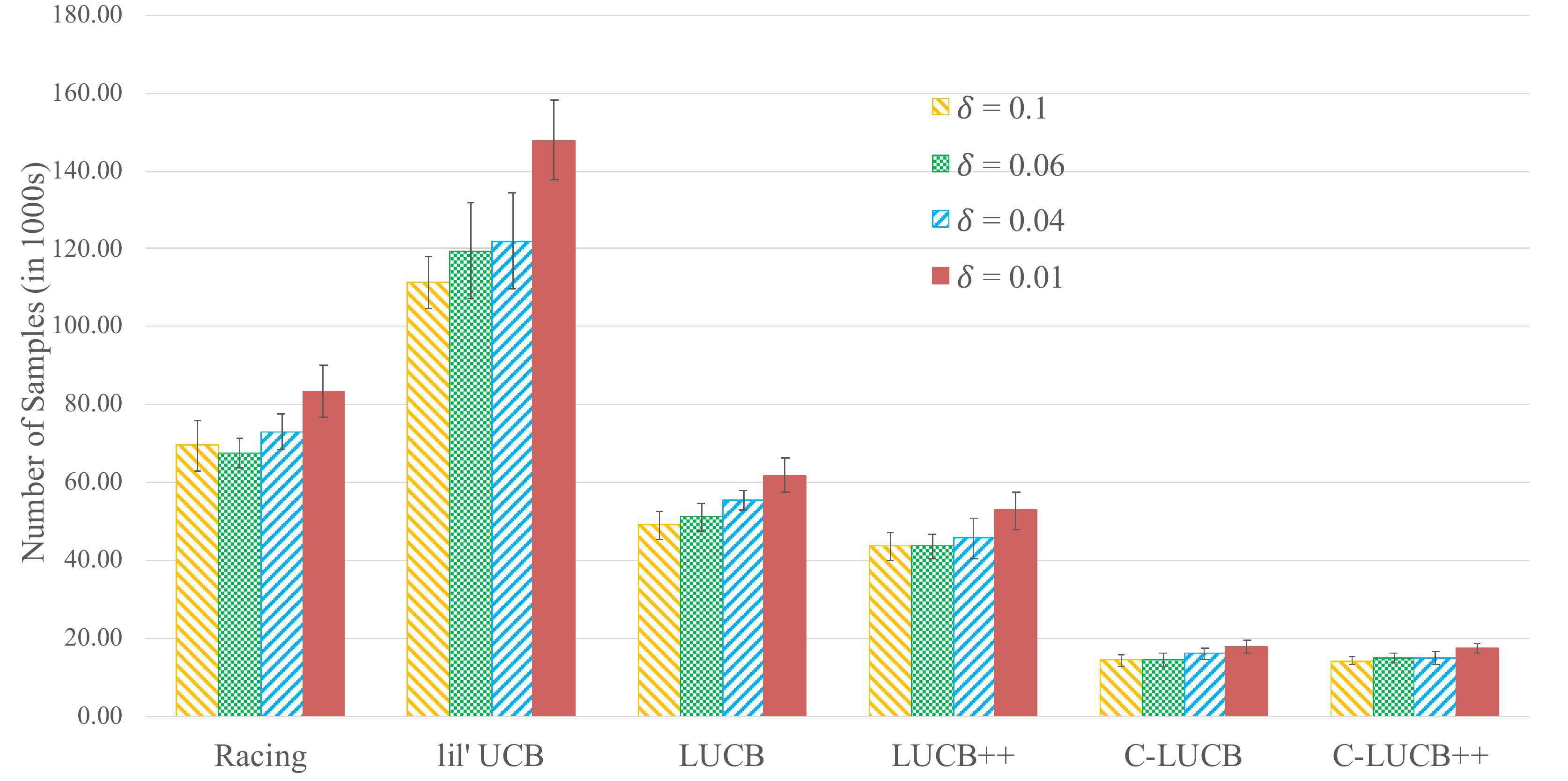}
    \caption{This plot illustrates the number of samples required by different algorithms to identify the best movie genre out of the 18 possible movie genres in the Movielens dataset with confidence $1 - \delta$. As $\delta$ decreases, the algorithms need more samples to identify the best arm. As our proposed C-LUCB and C-LUCB++ algorithms utilize correlation information, they identify the best arm in fewer samples relative to Racing, lil'UCB, LUCB and LUCB++. 
    }
    \label{fig:introFig}
\end{figure}

\vspace{0.2cm} \noindent
\textbf{Organization of the rest of the paper.}
In \Cref{sec:model} of this paper, we present a new multi-armed bandit framework, where correlation between arms is captured in the form of pseudo-rewards. We also discuss how pseudo-rewards can be computed in practical settings in \Cref{sec:model}. In \Cref{sec:priorWork}, we review state-of-the-art best-arm identification algorithms such as successive elimination (or racing), lil'UCB, and LUCB designed for the classical (independent arm) framework. We also discuss how our proposed correlated multi-armed bandit framework compares with the structured and linear bandit frameworks that have been studied previously. In \Cref{sec:algo} we propose the C-LUCB algorithm, and compare it with state-of-the-art approaches. We discuss several variants of C-LUCB in \Cref{sec:variants}. In \Cref{sec:regret} we analyze the sample complexity analysis of C-LUCB and discuss its proof technique and implications. 
This analysis reveals that utilizing correlations can lead to significant reduction in the number of samples required to identify the best-arm. Finally, in \Cref{sec:experiments} we demonstrate the practical applicability our proposed model and algorithm via extensive experiments on real-world recommendation datasets.

\section{The Correlated Multi-Armed Bandit Model}
\label{sec:model}
\subsection{Problem formulation}

\begin{table}[t]
\centering
\begin{tabular}{|l|l|l|l|l|}
\cline{1-2} \cline{4-5}
\textbf{r} & \textbf{$s_{2,1}(r)$} &  & \textbf{r} & \textbf{$s_{1,2}(r)$} \\ \cline{1-2} \cline{4-5} 
\textbf{0} & 0.7                   &  & \textbf{0} & 0.8                     \\ \cline{1-2} \cline{4-5} 
\textbf{1} & 0.4                   &  & \textbf{1} & 0.5                     \\ \cline{1-2} \cline{4-5} 
\end{tabular}
\\ \vspace{2mm}
\parbox{.45\linewidth}{
\centering
\begin{tabular}{|l|l|l|}
\hline
       & $R_1 = 0$ & $R_1 = 1$ \\ \hline
$R_2 = 0$ & 0.2       & 0.4       \\ \hline
$R_2 = 1$ &0.2 &0.2 \\ \hline
\end{tabular}
\\ \vspace{1mm} \textbf{(a)}
}
\hfill
\parbox{.45\linewidth}{
\centering
\begin{tabular}{|l|l|l|}
\hline
        & $R_1 = 0$ & $R_1 = 1$ \\ \hline
$R_2 = 0$ & 0.2       & 0.3       \\ \hline
$R_2 = 1$ & 0.4       & 0.1       \\ \hline
\end{tabular}
\\ \vspace{1mm}  \textbf{(b)}
}
\caption{The top row shows the pseudo-rewards of arms 1 and 2, i.e., upper bounds on the conditional expected rewards (which are known to the player). The bottom row depicts two possible joint probability distribution (unknown to the player). Under distribution (a), Arm 1 is optimal whereas Arm 2 is optimal under distribution (b). }
\label{tab:pseudoBin}
\vspace{-0.2cm}
\end{table}

Consider a Multi-Armed Bandit setting with $K$ arms $\{1,2, \ldots K\}$. At each round $t$, we sample an arm $k_t \in \mathcal{K}$ and receive a random reward $R_{k_t} \in [0,\B]$. Among the set of $K$ arms, we denote the arm with the largest mean reward as the \emph{best-arm} $k^*$, i.e., $k^* = \argmax_{k \in \mathcal{K}} \mu_k$. In the \emph{fixed-confidence} setting \cite{jamieson2014best}, the objective is to identify the best-arm in as few samples as possible. In particular, given $\delta>0$, the goal is to devise a sampling strategy that stops at some round $T$ (a random variable) and declares an arm $k^{\text{out}}$ as the optimal arm, where, 
$$\Pr(k^{\text{out}} = k^*) \geq 1 - \delta.$$ Put differently, we aim to find the best arm with probability at least $1-\delta$ while minimizing the total {\em number of samples} drawn from the arms. We note that the number of samples can be different from the number of rounds $T$ as some algorithms (e.g., LUCB, Racing) sample multiple arms in one round. Using the total number of samples drawn until round $T$ allows us to compare them fairly against algorithms that draw only one sample at each round $t$ 
(e.g., lil'UCB).

The classical multi-armed bandit setting implicitly assumes that the rewards $R_1, R_2, \ldots, R_K$ are independent. That is, $\Pr(R_{\ell} = r_\ell | R_k = r) = \Pr(R_{\ell} = r_\ell) \quad \forall{r_{\ell},r}$ and $\forall{\ell,k},$ which implies that, $\E{R_{\ell} | R_k = r} = \E{R_{\ell}} \quad \forall{r,\ell,k}$.  Motivated by 
the fact that rewards of a user corresponding to different arms might be correlated,
we consider a setup where 
$f_{R_\ell | R_{k}}(r_{\ell} | r_k) \neq f_{R_\ell}(r_{\ell})$, with $f_{R_\ell}(r_{\ell})$ denoting the probability distribution function of the reward from arm $\ell$. Consequently, due to such correlations, we have $\E{R_{\ell} | R_k} \neq \E{R_{\ell}}$.

In our problem setting, we consider that the player has partial knowledge about the joint distribution of correlated arms in the form of \emph{pseudo-rewards}, as defined below:

\begin{defn}[Pseudo-Reward]
Suppose we sample arm $k$ and observe reward $r$. Then the pseudo-reward of arm $\ell$ with respect to arm $k$, denoted by $s_{\ell,k}(r)$, is an upper bound on the conditional expected reward of arm $\ell$, i.e.,
\begin{equation}
 \mathbb{E}[R_\ell | R_k = r] \leq s_{\ell,k}(r).
\end{equation}
For convenience, we set $s_{\ell,\ell}(r) = r$.
\end{defn}

\begin{rem}
Note that the pseudo-rewards are upper bounds on the expected conditional reward and not hard bounds on the conditional reward itself. This makes our problem setup practical as upper bounds on expected conditional reward are easier to obtain, as illustrated below.
\end{rem}

The pseudo-reward information consists of a set of $K \times K$ functions $s_{\ell,k}(r)$ over $[0,\B]$. This information can be obtained in practice through either domain and expert knowledge or from controlled surveys. For instance, in the context of medical testing, where the goal is to identify the best drug to treat an ailment from among a set of $K$ possible options, the effectiveness of two drugs is correlated when the drugs share some common ingredients. Through domain knowledge of doctors, it is possible to answer questions such as ``what are the chances that drug $B$ would be effective given drug $A$ was not effective?", through which we can infer the pseudo-rewards. 

\begin{table}[t]
\parbox{.3\linewidth}{
\centering
Observation from Arm 1 \\
\vspace{2mm}
\begin{tabular}{|l|l|l|}
\hline
\textbf{r} & \textbf{$s_{2,1}(r)$} & \textbf{$s_{3,1}(r)$} \\ \hline
\textbf{0} & 0.7                   &  \textcolor{red}{2}                  \\ \hline
\textbf{1} & 0.8                   & 1.2                   \\ \hline
\textbf{2} & \textcolor{red}{2}                   & 1                   \\ \hline
\end{tabular}
}
\hfill
\parbox{.3\linewidth}{
\centering
Observation from Arm 2 \\
\vspace{2mm}
\begin{tabular}{|l|l|l|}
\hline
\textbf{r} & \textbf{$s_{1,2}(r)$} & \textbf{$s_{3,2}(r)$} \\ \hline
\textbf{0} & 0.5                   & 1.5                   \\ \hline
\textbf{1} & 1.3                   &       \textcolor{red}{2}              \\ \hline
\textbf{2} &    \textcolor{red}{2}                 &   0.8                 \\ \hline
\end{tabular}
}
\hfill
\parbox{.3\linewidth}{
\centering
Observation from Arm 3 \\
\vspace{2mm}
\begin{tabular}{|l|l|l|}
\hline
\textbf{r} & \textbf{$s_{1,3}(r)$} & \textbf{$s_{2,3}(r)$} \\ \hline
\textbf{0} & 1.5                    &        \textcolor{red}{2}             \\ \hline
\textbf{1} &   \textcolor{red}{2}                  & 1.3                   \\ \hline
\textbf{2} & 0.7                   & 0.75                   \\ \hline
\end{tabular}
}
\caption{
If some pseudo-reward entries are unknown (due to lack of domain knowledge), those entries can be replaced with the maximum possible reward and then used in the C-LUCB algorithm. We do that here by entering $2$ for the entries where pseudo-rewards are unknown.}
\label{tab:paddedEntries1}
\vspace{-0.2cm}
\end{table}

\vspace{0.1cm}
\noindent
\textbf{Computing Pseudo-Rewards from domain knowledge or historical data.} The pseudo-rewards can also be obtained from domain knowledge or through {\em offline} pilot surveys in which users are presented with {\em all} $K$ arms allowing us to sample $R_1, \ldots, R_K$ jointly. Through such data, we can evaluate an estimate on the conditional expected rewards. For example in \Cref{tab:pseudoBin}, we can look at all users who obtained $0$ reward for Arm 1 and calculate their average reward for Arm 2, say $\hat{\mu}_{2,1}(0)$. Since we only need an upper bound on $\E{R_2 | R_1 = 0}$, we can use any one of the following approaches to set the pseudo-reward $s_{2,1}(0)$.

\begin{enumerate}
    \item The pseudo-reward $s_{2,1}(0)$ can be set to $\hat{\mu}_{2,1}(0) + \hat{\sigma}_{2,1}(0)$, where $\hat{\mu}_{2,1}(0)$ is the empirical average of conditional rewards of $R_2$ given $R_1= 0$ and $\hat{\sigma}_{2,1}(0)$ is the empirical standard deviation. Adding the standard deviation ensures that the pseudo-reward is an upper bound on the conditional expected reward $\E{R_2 | R_1 = 0}$ with high probability.
    \item Alternately, pseudo-rewards for any unknown conditional mean reward could be set to $\B$, the maximum possible reward for the arm (recall that $R_k \in [0, \B]$). \Cref{tab:paddedEntries1} shows an example where unknown pseudo-rewards are set to $2$, the maximum possible reward. 
    \item  If through the training data, we obtain a soft upper bound $u$ on $\E{R_2|R_1 = 0}$ that holds with probability $1-\delta$, then we can translate it to the pseudo-reward $s_{2,1}(0) = u \times (1 - \delta) + 2 \times \delta$, (assuming maximum possible reward is 2). 
\end{enumerate}

\begin{rem}[Reduction to Classical Multi-Armed Bandits]
When all pseudo-reward entries are unknown, then all pseudo-reward entries can be filled with maximum possible reward for each arm, that is, $s_{\ell, k}(r) = \B$  $\forall{r,\ell,k}$. In that case, the problem framework studied in this paper reduces to the setting of the classical Multi-Armed Bandit problem.  
\end{rem}

While the pseudo-rewards are known in our setup, the underlying joint probability distribution of rewards is unknown. For instance, \Cref{tab:pseudoBin}(a) and \Cref{tab:pseudoBin}(b) show two joint probability distributions of the rewards that are both possible given the pseudo-rewards at the top of \Cref{tab:pseudoBin}.
If the joint distribution is as given in \Cref{tab:pseudoBin}(a), then Arm 1 is optimal, while Arm 2 is optimal if the joint  distribution is as given in \Cref{tab:pseudoBin}(b).

\subsection{Application for correlated multi-armed bandits}

Consider a scenario where a company needs to run a display advertising campaign \revadd{in a community} for one of their products, and their design team has proposed several different designs. The traction (i.e., the number of clicks, time spent on the ad) that the company generates is likely to be dependent on the design that is used for publicity. In order to find the best design, the company can run a best-arm identification algorithm by viewing the problem as a multi-armed bandit problem. Here, at each round $t$, a \revadd{new} user \revadd{of that community} enters the system and they show one of the $K$ designs (i.e., arms) to this user. The reward is received through the response of the user to the ad. A straightforward solution would be to treat this problem as a classical multi-armed bandit problem and use a well known best-arm identification algorithm such as lil'UCB, LUCB or successive elimination \revadd{to identify the best design for the community}. But, in practice, the rewards corresponding to different designs are likely to be correlated to one another. Consider the example shown in \Cref{fig:clooneyEx}, over there if a user reacts positively to the first design, the user is also likely to react positively to the second ad as both ads are related to tennis. Such correlations, when accounted for in the form of pseudo-rewards, can help us identify the best-arm in much fewer samples relative to algorithms such as lil'UCB, LUCB and Successive elimination that do not account for correlations in choices. 

These correlations could be known from a controlled survey or a previous advertisement campaign performed in a different demographic. For instance, from these surveys one can interpret information such as "users who like ad 1 representing tennis tend to like ad 2 that also represents tennis but not ad K which represents soccer". If a company wants to identify the best ad in a new demographic, it can use this learned correlation information to identify the best-ad in a quick manner. Note that the population composition in the two demographics may be very different, i.e., the fraction of users liking tennis may be very different, but it is likely that the correlation in choices remain consistent across the two demographics. One can also consider the example of identifying best policy to publicize for a political campaign, where users preferences towards different policies (i.e., climate change, gun control, abortion laws) are often correlated in all demographics, but the marginal distribution of people advocating for a single policy is very different in different communities. In such scenarios, transferring correlation information from one demographic to another by modeling them through pseudo-reward in our correlated bandit framework can help reduce the number of samples needed to identify the best-arm.

These pseudo-rewards can also be known from domain knowledge. Consider the problem of identifying the best drug for the treatment of an unknown disease. The effectiveness of different drugs is likely to be correlated as they often contain similar components. In such a situation, the domain expertise of doctors can tell us "what are the chances that drug y will be effective given drug x was effective?". One can use a conservative upper bound on the answer to this question to model pseudo-rewards. Alternatively, such correlation information could also be obtained on how different people react to different drugs in a community. As the effectiveness of drugs depends on underlying medical conditions of the patients, their response would be correlated. This correlation knowledge can then be transferred to identify the best treatment in a different community, where the distribution of underlying medical conditions may be very different.

\subsection{Special Case: Correlated Bandits with a Latent Random Source}
\label{subsec:specialCase}

\begin{figure}[t]
    \centering
    \includegraphics[width = 0.6\textwidth]{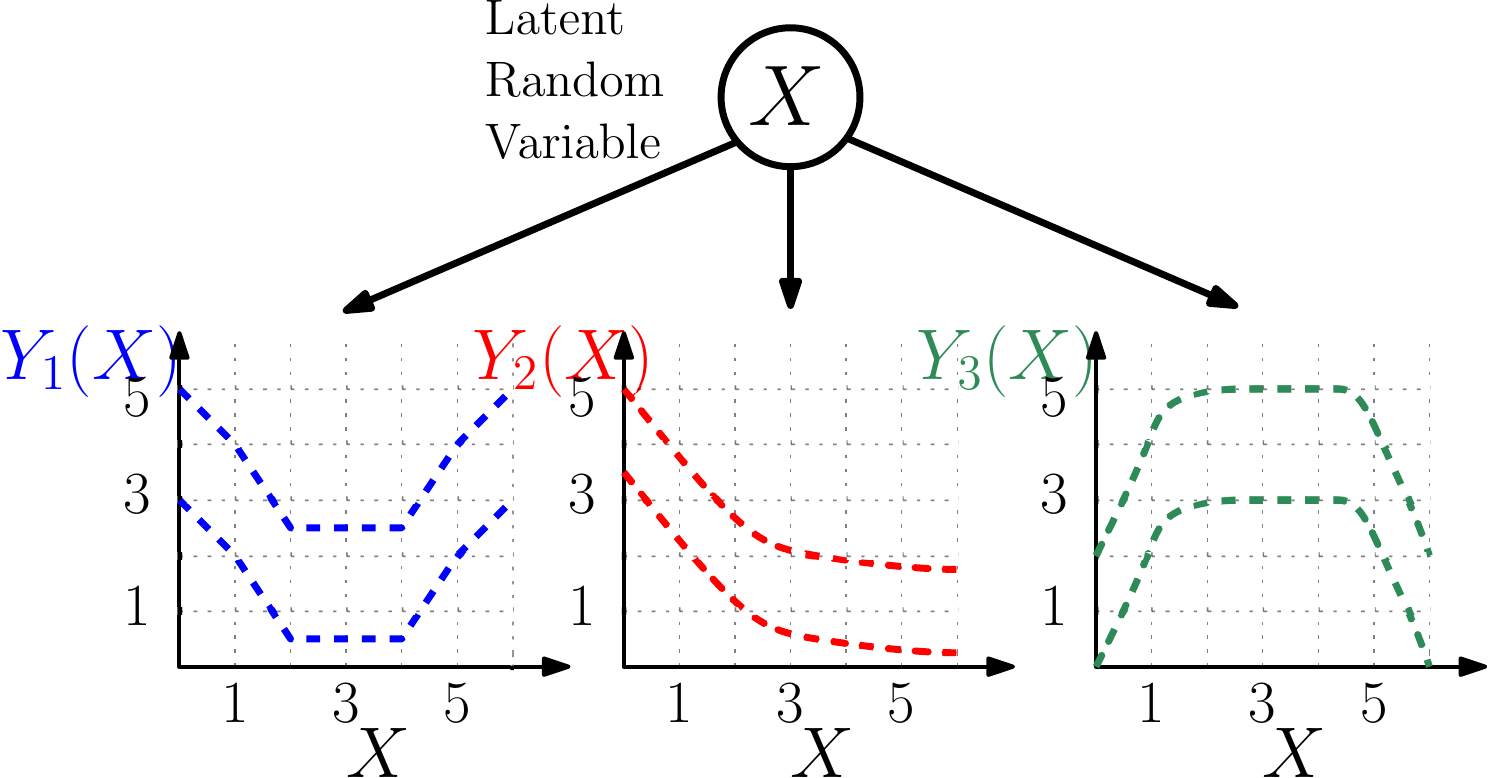}
    \caption{A special case of our proposed problem framework is a setting in which rewards for different arms are correlated through a hidden random variable X. At each round $X$ takes a realization in $\mathcal{X}$. The reward obtained from an arm $k$ is $Y_k(X)$. The figure illustrates lower bounds and upper bounds on $Y_k(X)$ (through dotted lines). For instance, when $X$ takes the realization $1$, reward of arm 1 is a random variable bounded between $2$ and $4$. }
    \label{fig:latentExample}
    \vspace{-0.2cm}
\end{figure}

The studied correlated multi-armed bandit can generalize several other interesting and unexplored multi-armed bandit problems. For example, one special case is the correlated multi-armed bandit model where rewards are correlated through a latent random source \cite{gupta2020correlated} (See \Cref{fig:latentExample}). In this problem setup, the hidden random variable $X$ takes an i.i.d. realization $X_t \in \mathcal{X}$ at round $t$ and upon pulling arm $k$ at round $t$, reward $Y_k(X_t)$ is observed. For the application setting of ad-recommendation, the random variable $X$ can represent the \textit{features} (i.e., age/occupation/income etc.) of the user. At each round a new user with feature $X_t$ enters the system, and the goal is to identify the single best ad recommendation for the whole population in as few samples as possible. The feature $X_t$ remains hidden to the player due to privacy concerns. Additionally, the reward $Y_k(X_t)$ represents the preference of the $k^{\text{th}}$ ad for the user with feature $X_t$. 

In this problem setup, the correlation information is known to the player in the form of upper and lower bounds on $Y_k(X)$, namely $\bar{g}_k(X)$ and $\underline{g}_k(X)$. These upper and lower bounds can be probabilistic, e.g., they may hold with probability $0.8$ ($80\%$ confidence). For instance, the information on prior information represents the knowledge that \textit{children of age 5-10 rate documentaries only in the range 1-3 out of 5 in $80\%$ cases}. While such prior knowledge may be known from domain expertise or previous ad-campaigns performed in a different demographic, the age distribution of the community may be unknown. Due to which, the best-arm remains unknown and it needs to be found in an online manner.

This particular correlated bandit setting can be reduced to our general framework by translating the mappings $Y_k(X)$ to pseudo-rewards $s_{\ell,k}(r)$. Recall the pseudo-rewards represent an upper bound on the conditional expectation of the rewards. In this framework, if  $\underline{g}_k(x)$ and $\bar{g}_k(x)$ are soft lower and upper bounds, i.e., $\underline{g}_k(x) \leq Y_k(x) \leq \bar{g}_k(x)$ w.p. $1 - \kappa$, we can construct pseudo-reward as follows: 
\begin{equation}
\label{eqn:psrX}
    s_{\ell,k}(r) = (1 - \kappa)^2 \times \left( \max_{\{x: \underline{g}_k(x) \leq r \leq \bar{g}_k(x)\}}  \bar{g}_{\ell}(x)  \right) + (1 - (1 - \kappa)^2) \times M,
\end{equation}
 where $M$ is the maximum possible reward an arm can provide. We evaluate this pseudo-reward by first finding the range of values within which $x$ lies based on the reward with probability $1-\kappa$. The maximum possible reward of arm $\ell$ for values of $x$ is then identified with probability $1-\kappa$. Due to this, with probability $(1 - \kappa)^2$, conditional reward of arm $\ell$ is at-most $\max_{\{x: \underline{g}_k(x) \leq r \leq \bar{g}_k(x)\}}  \bar{g}_{\ell}(x)$. As the maximum possible reward is $M$ otherwise, we get the pseudo-reward as shown in \eqref{eqn:psrX}. Once these pseudo-rewards are constructed, the problem fits in the general framework described in this paper and we can use the algorithms proposed for this setting directly.

The presented model resembles the structured bandit model studied in \cite{lattimore2014bounded} in which mean rewards of different arms, $\mu_k(\theta)$, are known as a function of a hidden parameter $\theta$, but the parameter $\theta$ is unknown. It is important to see that this presented model differs from \revadd{\cite{lattimore2014bounded}} in two key ways -- i) In \revadd{\cite{lattimore2014bounded}}, instead of a hidden random variable $X$, there is a hidden feature $\theta$ which is fixed and unknown and ii) the mean reward mappings as a function of $\theta$ are known, whereas in our model we consider the knowledge of soft upper and lower bounds on $Y_k(X)$. The model studied in \revadd{\cite{lattimore2014bounded}} is more suitable for settings where the goal is to provide personalized recommendation to a user whose features $\theta$ are hidden, whereas the latent random source model (and the general correlated bandit model) is appropriate for application settings where the goal is to identify a single recommendation for the global demographic.

Note that the model presented in this subsection requires the understanding of hidden random variable $X$. While in certain problem settings it may be possible to obtain a latent random source representation in the form of $X$. In general, these hidden features may be more complicated and one may not be able to represent them. It is important to note that our proposed model in the most general setting works without having to construct a hidden feature representation through which arms are correlated. This is a key advantage of our general model over the latent random source model and the model presented in \cite{lattimore2014bounded}, which requires modeling the problem through a hidden parameter $\theta$. Instead, our general model utilizes the available prior information directly and our algorithms adapt to the information to identify the best-arm in fewer samples relative to classical best-arm identification algorithms.

\section{Related Prior Work}
\label{sec:priorWork}
The design of best-arm identification algorithms in the fixed-confidence setting have three key design components:  i) their sampling strategy, i.e., which arm to pick at round $t$; ii) their elimination criteria, i.e., when to declare an arm as sub-optimal and remove it from the rest of the sampling procedure; and  iii) their stopping criteria, i.e., when to stop the algorithm and declare an arm as the best arm. 

In order to accomplish the task of best-arm identification, algorithms  use the empirical mean $\hat{\mu}_k(t)$ for arm $k$ at round $t$. In addition to this, upper confidence bound and lower confidence bound on the mean of arm $k$ are maintained based on the number of samples of arm $k$, $n_k(t)$, and the input confidence parameter $\delta$. In particular, the upper confidence index $U_k(n_k, \delta) = \hat{\mu}_k(t) + \Confidence(n_k, \delta)$ and lower confidence index $L_k(n_k, \delta) = \hat{\mu}_k(t) - \Confidence(n_k, \delta)$ are maintained for each arm $k \in \mathcal{K}$. Here $\Confidence(n_k, \delta) \propto \sqrt{\frac{\log\left(\frac{\log(n_k)}{\delta}\right)}{n_k}}$ is an {\em anytime} confidence bound \cite{jamieson2014lil, howard2018timeuniform} constructed such that 
\begin{equation}
\Pr\Big(\exists ~ n_k \geq 1: ~~ \mu_k \notin [ L_k(n_k, \delta), ~ U_k(n_k, \delta) ]\Big) \leq \delta.
\label{eq:anytime}
\end{equation}

Note that the anytime confidence interval bound the probability of the mean lying outside the confidence interval uniformly for all $n_k \geq 1$, i.e., the probability that the mean lies outside the confidence interval $[ L_k(n_k, \delta), ~ U_k(n_k, \delta) ]$ at \emph{any} round $t$ is upper bounded by $\delta$. In contrast to the Hoeffding bound, which are only valid for a fixed and deterministic $n_k$, the anytime confidence bound holds true uniformly for all $t \geq 1$ and for random $n_k$ as well. We refer the reader to \cite{howard2018timeuniform} for a detailed discussion and developments in anytime confidence bounds $B(n_k, \delta)$.

\begin{table}[t]
\resizebox{\textwidth}{!}{
\centering
\begin{tabular}{l l l l }
\hline
\textbf{Algorithm} & \textbf{Sampling Strategy} & \textbf{Eliminate Arm $k$ if} &\textbf{Stopping Criteria}\\ \hline
Racing & Round Robin in $\mathcal{A}_t$                   &    $U_k\left(\frac{\delta}{K}\right) < \max\limits_{\ell \in \mathcal{A}_t} L_\ell\left(\frac{\delta}{K}\right)$   & $|\mathcal{A}_t| = 1$             \\ \hline
lil'UCB & Sample $k_t$,~ $k_t = \argmax_{k} U_k( \delta)$                   & N/A & $n_{k_t} \geq \alpha \sum_{k \neq k_t} n_k$                 \\ \hline
LUCB & Sample $m_1, m_2$,  & $U_k\left(\frac{\delta}{K}\right) < \max\limits_{\ell \in \mathcal{A}_t} L_\ell\left(\frac{\delta}{K}\right)$  &    $|\mathcal{A}_t| = 1$* or              \\ 
 &       $m_1 = \argmax\limits_{k \in \mathcal{A}_t} \hat{\mu}_k(t),$   &   & $L_{m_1}\left(\frac{\delta}{K}\right) > U_{m_2}\left(\frac{\delta}{K}\right)$                 \\
 & $m_2 = \argmax\limits_{k \in \mathcal{A}_t \setminus \{ {m_1}\}} U_k\left(\frac{\delta}{K}\right)$ & & \\
 \hline
LUCB++ & Sample $m_1, m_2$,                    &    &  $L_{m_1}\left(\frac{\delta}{2K}\right) > U_{m_2}\left(\frac{\delta}{2}\right)$            \\ 
 &       $m_1 = \argmax\limits_{k \in \mathcal{K}} \hat{\mu}_k(t)$   &  N/A  &                  \\ 
 & $m_2 = \argmax\limits_{k \in \mathcal{K} \setminus \{ {m_1}\}} U_k\left(\frac{\delta}{2}\right)$& & \\\hline
 \textbf{C-LUCB} & Sample $m_1, m_2$,                    & $\tilde{U}_k\left( \frac{\delta}{2K}\right) < \max\limits_{\ell \in \mathcal{A}_t} L_\ell\left( \frac{\delta}{2K}\right)$  &    $|\mathcal{A}_t| = 1$               \\ 
\textbf{(ours)}  &    $m_1 = \argmax\limits_{k \in \mathcal{A}_t} I_k(t)$,   &   &                  \\
& $m_2 = \argmax\limits_{k \in \mathcal{A}_t \setminus \{ {m_1}\}}  \min\left(\tilde{U}_{k,k}\left(\frac{\delta}{2K}\right), I_k(t)\right)$ & & \\ \hline
\textbf{C-LUCB++} & Sample $m_1, m_2$,                    &  $\tilde{U}_k\left( \frac{\delta}{3K}\right) < \max\limits_{\ell \in \mathcal{A}_t} L_\ell\left( \frac{\delta}{3K}\right)$  &  $|\mathcal{A}_t| = 1$ or          \\ 
 \textbf{(ours)} &   $m_1 = \argmax\limits_{k \in \mathcal{A}_t} I_k(t)$,       &   &  $L_{m_1}\left(\frac{\delta}{4K}\right) > \tilde{U}_{m_2, m_2}\left(\frac{\delta}{4}\right)$                \\ 
 &$m_2 = \argmax\limits_{k \in \mathcal{A}_t \setminus \{ {m_1}\}} \min\left(\tilde{U}_{k,k}\left(\frac{\delta}{2}\right), I_k(t)\right)$ & & \\ \hline 
\end{tabular}
}
\caption{All best-arm identification algorithms have three key components, i) Sampling strategy at each round $t$, ii) elimination criteria for an arm and iii) the stopping criteria of the algorithm. We compare these for Racing, lil'UCB, LUCB and LUCB++ algorithms and see the differences in their operation. The indices used for our proposed C-LUCB and C-LUCB++ are defined in \eqref{eq:crossUCBdefn} and \eqref{eq:pseudoUCBIndexdefn}.}
\label{tab:summaryAlgoClassic}
\vspace{-0.2cm}
\end{table}

\subsection{Existing Best-Arm identification strategies}

There are three well-known approaches to the best-arm identification problem: i) Successive Elimination (also called racing) \cite{bechhofer1958sequential, paulson1964sequential, even2002pac}; ii) lil'UCB (Law of Iterated Logarithms Upper Confidence Bound) \cite{jamieson2014lil}; and iii) LUCB \cite{kalyanakrishnan2012pac, kaufmann2013information} (Lower and Upper Confidence Bound). 
Below, we briefly introduce these algorithms, and present
a summary of their arm sampling strategies and elimination and stopping criteria   in \Cref{tab:summaryAlgoClassic} \footnote{The confidence bound $C(n_k(t), \delta)$, and subsequently lower and upper confidence indices $L_k(n_k(t),\delta)$ and $U(n_k(t), \delta)$, depend on the number of rounds $t$, the number of samples of arm k till round t $n_k(t)$ and the confidence parameter $\delta$. For brevity purposes, at times we represent the confidence bound as  $C(n_k, \delta)$ or $C(\delta)$ and the LCB, UCB indices as $L_k(t, \delta), L_k(n_k, \delta)$ or $L_k(\delta)$ and $U_k(t, \delta), U_k(n_k, \delta)$ or $U_k(\delta)$ respectively.}. For more details, we refer the reader to \cite{jamieson2014best} that provides a comprehensive survey of best-arm identification in the fixed confidence setting.

\vspace{0.2cm}
\noindent
\textbf{Successive Elimination or Racing:} The successive elimination (also called racing) strategy maintains a set of active arms $\mathcal{A}_t$ at each round. It samples arms in a round-robin fashion from the set of active arms and at the end of each round, it eliminates an arm $k$ from the set of active arms if the lower confidence index of some other arm $\ell \neq k$, $L_\ell\left(n_\ell, \frac{\delta}{K}\right)$, is strictly larger than the upper confidence index of arm $k$, $U_k\left(n_k, \frac{\delta}{K}\right)$. It continues this until a single arm is left in the set $\mathcal{A}_t$ and returns that arm as the optimal arm. Two other algorithms, Exponential-gap elimination \cite{karnin2013almost} and PRISM \cite{jamieson2013finding}, build upon successive elimination to provide stronger theoretical guarantees. However, their empirical performance is not promising as noted in \cite{jamieson2014best}. 

\noindent
\textbf{lil'UCB \cite{jamieson2014lil}:} The lil'UCB algorithm samples the arm with the largest upper confidence index $U_k(n_k, \delta)$ at round $t$ and stops when an arm has been sampled more than $\frac{\alpha t}{\alpha + 1}$ times till round $t$. In practice, the value of $\alpha$ is taken to be $9$. It then declares the most sampled arm as the best-arm. 

\noindent
\textbf{LUCB \cite{kalyanakrishnan2012pac, jamieson2014best}:} The LUCB approach samples two arms $m_1(t), m_2(t)$ at each round $t$. Here, $m_1(t)$ is the arm with the largest empirical reward till round $t$, and  $m_2(t)$ is the arm with the largest UCB index $U_k\left(n_k, \frac{\delta}{K}\right)$ among the rest. The LUCB algorithm stops if the lower confidence bound of the first arm $m_1(t)$ is larger than the upper confidence index of all other arms. \footnote{Equivalently, one can eliminate an arm $k$ from $\mathcal{A}_t$ at the end of each round if the upper confidence index of arm $k$ is smaller than the lower confidence index of some other arm, and stop the algorithm when the set of active arms $|\mathcal{A}_t| = 1$. This implementation of the LUCB algorithm has the same guarantees as the one proposed in \cite{kalyanakrishnan2012pac, jamieson2014best} while obtaining similar empirical performance.} Subsequently, another algorithm LUCB++ \cite{simchowitz2017simulator, tanczos2017kl} was designed that operates in a similar manner to LUCB but constructs the upper confidence and lower confidence indices with different confidence parameters for $m_1(t), m_2(t)$. The details of the upper confidence and lower confidence indices for each of these algorithms are presented in \Cref{tab:summaryAlgoClassic}. Note that our metric for comparison is the total number of samples collectively drawn from the arms. As LUCB algorithms sample two arms at each round, the total number of samples drawn from the LUCB algorithms is two times the number of rounds $t$. By comparing the total number of samples and not the number of rounds $t$, we draw a fair comparison between the performance of LUCB and lil'UCB algorithm.

All the approaches described above work well for the case where rewards are known to be either sub-Gaussian or bounded. Furthermore, if the class of distribution is known (e.g., it is known that rewards are Gaussian with known $\sigma$ and unknown $\mu$), then there are two more approaches known in the literature, namely Top Two Thompson Sampling (TTTS) \cite{shang2020fixed} and Tracking \cite{garivier2016optimal}. In TTTS, the player computes a posterior distribution on the mean reward of each arm and then applies Thompson sampling on the posterior to obtain two samples. It stops when the posterior probability of an arm $k$ being optimal exceeds a certain threshold $\tau_k(n_k, \delta)$. The TTTS algorithm can be computationally intensive as it involves the computation of posterior probability in each round of their algorithm. In \cite{garivier2016optimal}, authors evaluate a lower bound for the Multi-Armed bandit problem in the form of an optimization problem. They propose a tracking based approach, that solves the optimization problem at each round to obtain an estimated rate at which each arm should be sampled at round $t$ and sample arms in proportion to that rate. \newadd{More recently, \cite{degenne2019non} proposed alternative approaches to the track-and-stop algorithm that do not require solving an optimization problem at each round. Instead, they view the optimization problem as an unknown game and have sampling rules based on iterative saddle point strategies. All of the approaches listed above require knowing the {\em class} of reward distribution.} Since we only assume that the rewards are bounded and not the class of distribution, we do not focus on extending TTTS or Tracking based approaches to the correlated bandit setting in this paper.

\begin{table}[t]
\resizebox{\textwidth}{!}{
\centering
\begin{tabular}{ l l l l }
\hline
\textbf{Algorithm} & \textbf{Confidence Bound $\Confidence(n_k, \delta)$} & \textbf{Type} &\textbf{Samples Drawn}\\ \hline
Succ Elimination \cite{even2002pac} & $\sqrt{\frac{\log\left(\frac{\pi^2 n_k^2}{3\delta}\right)}{2n_k}}$  & Racing   &   577209.4            \\ \hline
lil Succ Elimination \cite{jamieson2014best} & $0.85 \sqrt{\frac{\log(\log(0.2585 n_k)) + 0.96 \log(67.59/\delta)}{n_k}}$  & Racing   &   120498.5            \\ \hline
KL-Racing \cite{kaufmann2013information} & $d(\Confidence)= 2\log\left(\frac{11.1 t^{1.1}}{\delta}\right)$*  & Racing   &   147780.4             \\ \hline
Racing with \cite{howard2018timeuniform} & $0.85 \sqrt{\frac{\log(\log(0.5n_k)) + 0.72 \log(5.2/\delta)}{n_k}}$   & Racing   &   82504.7             \\ \hline
LUCB with \cite{kaufmann2013information} & $\sqrt{\frac{\log\left(\frac{405 t^{1.1}}{\delta} \log\left(\frac{405 t^{1.1}}{\delta}\right)\right)}{2n_k}}$  & LUCB   &   219510.2             \\ \hline
lil LUCB \cite{jamieson2013finding} & $0.85 \sqrt{\frac{\log(\log(0.2585 n_k)) + 0.96 \log(67.59/\delta)}{n_k}}$  & LUCB   &    90523.0             \\ \hline
KL-LUCB \cite{kaufmann2013information} & $\small{d(\Confidence) =} 2\log\left(\frac{405.5 t^{1.1}}{\delta}\right) + \log\log\left(\frac{405.5t^{1.1}}{\delta}\right)$  & LUCB   &   81154.4             \\ \hline
LUCB with \cite{howard2018timeuniform} &  $0.85 \sqrt{\frac{\log(\log(0.5n_k)) + 0.72 \log(5.2/\delta)}{n_k}}$ & LUCB   &   62533.2             \\ \hline
lil'UCB \cite{jamieson2014lil} & $0.85 \sqrt{\frac{\log(\log(0.2585 n_k)) + 0.96 \log(67.59/\delta)}{n_k}}$  & lil'UCB   &   140987.0             \\ \hline
lil-KL-LUCB \cite{tanczos2017kl} &   $d(\Confidence) = 1.86\log\left(\kappa \log_2\left(\frac{2n_k}{\delta}\right)\right)$ & LUCB++   &   92000.0             \\ \hline
LUCB++ with \cite{howard2018timeuniform} &  $0.85 \sqrt{\frac{\log(\log(0.5n_k)) + 0.72 \log(5.2/\delta)}{n_k}}$  & LUCB++   &   55138.8             \\ \hline
\end{tabular}
}
\caption{Description of the well-known best-arm identification algorithms and the confidence bound $\Confidence(n_k, \delta)$ that they use for [0,1] bounded rewards. All the three types of algorithms have evolved with time due to the development of tighter $1 - \delta$ anytime confidence intervals $\Confidence(n_k, \delta)$. We see that the algorithms perform best with the confidence bound suggested in \cite{howard2018timeuniform}, and hence we use that for all our implementations of Racing, LUCB, LUCB++ and our proposed algorithm in the rest of the paper. The reported sample complexity is for the task of identifying best movie genre from the set of 18 movie genres in the Movielens dataset. Experimental setup is described in detail in \Cref{sec:experiments}. }
\label{tab:ConfIntervals}
\vspace{-0.2cm}
\end{table}

\subsection{Developments in Confidence sequence $\Confidence(n_k, \delta)$}
It is important to note that the performance of the algorithms described above depends critically on the tightness of the confidence bound $\Confidence(n_k, \delta)$. For instance, initially the LUCB algorithm was proposed with the confidence interval $\Confidence(n_k, \delta) = \sqrt{\frac{\log\left(\frac{405 n_k^{1.1}}{\delta} \log\left(\frac{405 n_k^{1.1}}{\delta}\right)\right)}{2n_k}}$ (See \cite{kalyanakrishnan2012pac}) for $[0,1]$ bounded random variables. Subsequently tighter bounds as in \cite{jamieson2014best}, \cite{kaufmann2013information} were developed, which led to performance improvements in the LUCB algorithm. See \Cref{tab:ConfIntervals} for a comparison different confidence bound developed over time and how they affect the empirical performance of the best-arm identification algorithms\footnote{The bound proposed in \cite{kaufmann2013information, tanczos2017kl} are KL based bounds that evaluate the indices $U_k(n_k, \delta) , L_k(n_k, \delta)$ as $\inf\{j > \hat{\mu}_k: n_k(t)d_{kl}(\hat{\mu}_k, j) < d(B)\}$ and $\sup\{j < \hat{\mu}_k: n_k(t)d_{kl}(\hat{\mu}_k, j) < d(B)$. The distance $d_{kl}(x,y)$ is evaluated as $x \log(x/y) + (1-x) \log((1-x)/(1-y))$}. For a more detailed comparison of different confidence bounds $\Confidence_{k}(n_k, \delta)$, we refer the reader to Table 2 of \cite{howard2018timeuniform}. To the best of our knowledge, the tightest $1-\delta$ anytime confidence interval for bounded and sub-Gaussian random variables is proposed in \cite{howard2018timeuniform}, which constructs \begin{equation}
\Confidence(n_k, \delta) = 0.85 \sqrt{\frac{\log(\log(0.5n_k)) + 0.72 \log(5.2/\delta)}{n_k}}.
\label{eq:Ckdefn1}
\end{equation} 
Due to this observation, which is also supported by empirical evidence in \Cref{tab:ConfIntervals}, we use the bound suggested by \cite{howard2018timeuniform} in all  implementations of Successive Elimination, LUCB and our proposed algorithm. However, our algorithm and analysis extend to arbitrary $1 - \delta$ anytime confidence interval $\Confidence(n_k, \delta)$.

We would also like to highlight the fact that lil'UCB is known to have the best known theoretical sample complexity (in terms of its dependency on the number of arms $K$). The LUCB algorithm stops with probability $1 - \delta$ after obtaining at most $\sum_{k \in \mathcal{K}} \frac{2\constant}{\Delta_k^2} \left(\log\left(\frac{K\log\left(\frac{1}{\Delta_k^2}\right)}{\delta}\right)\right)$ samples, where $\Delta_k = \mu_{k^*} - \mu_k$, the difference in mean reward of optimal arm $k^*$ and mean reward of arm $k$. And $\Delta_{k^*} = \min_{k \neq k^*} \Delta_k$, the gap between best and second best arm. It is known that lil'UCB algorithm has a sample complexity $\OO\left(\sum_{k \in \mathcal{K}} \frac{1}{\Delta_k^2}\log\left(\frac{\log\left(\frac{1}{\Delta_k^2}\right)}{\delta}\right)\right)$i.e., it avoids the $\log(K)$ term in the numerator, and hence has the best known theoretical sample complexity. However, it has been observed (both in \cite{jamieson2014best} and our experiments) that its empirical performance is inferior to that of the LUCB algorithm. Due to this reason, we focus on proposing an algorithm C-LUCB that extends the LUCB approach to the correlated bandit setting. We have included the performance of lil'UCB in all our experiments.

\subsection{Algorithms outside the classical setting}
Unlike the regret-minimization problem, the best-arm identification problem is relatively unexplored outside of the classical multi-armed bandit setting. 
A rare exception is the {\em structured} bandit setting,
where mean rewards corresponding to different arms are related to one another through a hidden parameter $\theta$. The underlying value of $\theta$ is fixed and unknown, but the mean reward mappings $\theta \rightarrow \mu_k(\theta)$ are known. The {\em linear} bandit setting is a special case of  structured bandits, where mean reward mappings are of the form $x_k^\intercal \theta$ with $x_k$  known to the player.
The best-arm identification problem has been 
studied  in  \cite{soare2014best, tao2018best} for {\em linear} bandits and in \cite{huang2017structured} for the general {structured} bandit setting.  
Other special cases of  structured bandits  include global bandits \cite{ata2015global}, regional bandits \cite{wang2018regional} and the generalized linear bandits \cite{shen2018generalized}; to the best of our knowledge the best arm identification problem has not been addressed in these special cases. \newadd{Note that in the full generality, the structured bandit framework is simply a bandit problem with constraints on the joint probability distribution \cite{van2020optimal}, but that setting has only been studied for the objective of regret minimization and not best-arm identification. To the best of our knowledge, the structured bandits work studying best-arm identification \cite{soare2014best, tao2018best, huang2017structured} assume the presence of a hidden parameter $\theta$ through which mean rewards of different arms are related to one another. \revadd{Our correlated bandit framework focuses on structured bandit settings by modeling the correlations explicitly through the knowledge of pseudo-rewards.}}

\newadd{Recently, best-arm identification was studied under the spectral bandit framework \cite{kocak2020best}, which assumes that the arms are the nodes of known a weighted graph, with $w_{a,b}$ denoting the weight between arms $a$ and arms $b$. The spectral bandit framework poses a restriction on the relationship between mean rewards of individual arms by assuming that $\sum_{a, b \in \mathcal{K}} w_{a,b} \frac{(\mu_a - \mu_b)^2}{2} \leq R$, where $R$ is known to the player.}

The correlated bandit model considered in this paper is fundamentally different from the structured bandit framework as detailed below.
\begin{enumerate}
    \item The model studied here explicitly models the correlations in the rewards of different arms {\em at any given round $t$}. In structured bandits, the mean rewards are related to each other, but the reward realizations at a given round are not necessarily correlated. \newadd{Similar to structured bandits, the work on spectral bandits \cite{kocak2020best} considers a setup with constrains between mean rewards of different arms, but does not capture the correlations explicitly in their framework. }
    \item It is also possible to use the structured bandit framework for the objective of identify best global recommendation in an ad-campaign. However, there are two major challenges i) In deciding upon the hidden parameter $\theta$ that we need to use, through which the mean rewards are related to one another. ii) Secondly, in the structured bandits framework, the reward mappings from $\theta$ to $\mu_k(\theta)$ need to be {\em exact}. If they happen to be incorrect, then the algorithms for structured bandit cannot be used as they rely on the correctness of $\mu_k(\theta)$ to construct confidence intervals on the unknown parameter $\theta$. In contrast, the model studied here only relies on the pseudo-rewards being upper bounds on the conditional expectations $\E{R_{\ell}|R_k = r}$. Our proposed algorithm works even when these bounds are not tight. The lack of hidden parameter $\theta$ and pseudo-rewards being upper bounds on conditional expectations make the model studied in this paper more suitable for practical scenarios where the goal is to identify the best global recommendation. 
\end{enumerate}

\section{Proposed Correlated-LUCB Best-arm Identification Algorithm}
\label{sec:algo}

In the correlated MAB framework, the rewards observed from one arm can help estimate the rewards from other arms. Our key idea is to use this information to reduce the number of samples taken before stopping. We do so by maintaining the \emph{empirical pseudo-rewards} of all pairs of distinct arms at each round $t$.

\subsection{Empirical Pseudo-Rewards and New UCB indices}
\label{sec:pseudo_reward}

In our correlated MAB framework, pseudo-reward of arm $\ell$ with respect to arm $\arm$ provides us an estimate on the reward of arm $\ell$ through the reward sample obtained from arm $\arm$. We now define the notion of empirical pseudo-reward  which can be used to obtain an \textit{optimistic estimate} of $\mu_\ell$ through just reward samples of arm $\arm$.

\begin{defn}[Empirical and Expected Pseudo-Reward]
\label{defn:empirical_pseudo_reward}
After $\slot$ rounds, arm $\arm$ is sampled $\pulls_\arm(\slot)$ times. Using these $\pulls_\arm(\slot)$ reward realizations, we can construct the empirical pseudo-reward $\estimateMean_{\ell, \arm}(\slot)$ for each arm $\ell$ with respect to arm $\arm$ as follows. 
\begin{align}
\estimateMean_{\ell, \arm}(\slot) \triangleq \frac{\sum_{\tau=1}^{\slot} \indicator_{k_\tau = k} \ \estimateReward_{\ell, \arm}(\reward_{k_\tau})}{\pulls_\arm(\slot)}, \qquad \ell \in \{1,\ldots, K\} \setminus \{k\}. 
\end{align}
The expected pseudo-reward of arm $\ell$ with respect to arm $\arm$ is defined as
\begin{align}
\expectedPseudoReward_{\ell, \arm} \triangleq \E{\estimateReward_{\ell, \arm}(R_k)}.
\end{align}
For convenience, we set $\hat{\phi}_{k,k}(t) = \hat{\mu}_k(t)$ and $\phi_{k,k} = \mu_k$. Note that the empirical pseudo-reward $\estimateMean_{\ell, \arm}(\slot)$ is defined with respect to arm $\arm$ and it is only a function of the rewards observed by sampling arm $\arm$. 
\end{defn}

Observe that $\E{s_{\ell,k}(R_k)} \geq \E{\E{R_\ell | R_k=r}} = \mu_\ell$. Due to this, empirical pseudo-reward $\estimateMean_{\ell, \arm}(\slot)$ can serve as an estimated upper bound on $\mu_{\ell}$. Using the definitions of empirical pseudo-reward, we now define auxiliary UCB indices, namely crossUCB and pseudoUCB indices, which are used in the selection and elimination strategy of the C-LUCB algorithm.

\begin{defn}[CrossUCB Index $\tilde{U}_{\ell, k}(t, \delta)$]
At the end of round $t$, we have $n_k(t)$ samples of arm $k$. Using these, we define the CrossUCB Index of arm $\ell$ with respect to arm $k$ as 
\begin{equation}
    \tilde{U}_{\ell,k}(t, \delta) \triangleq \hat{\phi}_{\ell,k}(t) + \Confidence(n_k, \delta).
\label{eq:crossUCBdefn}
\end{equation}
Furthermore, we define $$\tilde{U}_\ell(t, \delta) = \min_{k} \tilde{U}_{\ell,k}(t, \delta),$$ i.e., the tightest of the $K$ upper bounds, $\tilde{U}_{\ell,k}(t, \delta)$, for arm $\ell$.
\end{defn}

\noindent
Note that the CrossUCB index for arm $\ell$ with respect to arm $k$, $\tilde{U}_{\ell,k}(t,\delta)$ is constructed only through the samples obtained from arm $k$. Furthermore, we have $\tilde{U}_{k,k}(t, \delta) = \hat{\mu}_k(t) + \Confidence(n_k, \delta)$, which coincides
with the standard upper confidence index used in the best-arm identification literature. We use the confidence bound suggested by \cite{howard2018timeuniform} (see \Cref{sec:priorWork}) for the construction of $\Confidence(n_k, \delta)$ for $[0,\B]$ bounded random variables, i.e., 

\begin{equation}
\Confidence(n_k, \delta) = \frac{1.7\B}{2}\sqrt{\frac{\log\left(\log\left(\frac{\B^2 n_k}{2}\right)\right) + 0.72 \log(5.2/\delta)}{n_k}}.
\label{eq:Ckdefn}
\end{equation}

\noindent
As pseudo-rewards are \emph{upper bounds} on conditional expected reward, they can only be used to construct alternative upper bounds on the mean reward of other arms and not alternative lower bounds. Due to this reason, we keep the definition of lower confidence index $L_k(t, \delta)$ the same as that in the classical multi-armed bandit setting, i.e., $L_k(t, \delta) = \hat{\mu}_k(t) - \Confidence(n_k, \delta)$. In addition to the CrossUCB and the LCB index for each arm, we now define the PseudoUCB index of arm $\ell$ with respect to arm $k$. The PseudoUCB indices prove useful for the design and analysis of our proposed algorithm.

\begin{defn}[PseudoUCB Index $I_{\ell,k}(t)$]
We define the PseudoUCB Index of arm $\ell$ with respect to arm $k$ as follows. 
\begin{equation}
    I_{\ell,k}(t) \triangleq \hat{\phi}_{\ell,k}(t) + \B\sqrt{\frac{2 \log t}{n_k(t)}}
\label{eq:pseudoUCBIndexdefn}
\end{equation}
Furthermore, we define $I_\ell(t) = \min_{k} I_{\ell,k}(t)$, the tightest of the $K$ upper bounds for arm $\ell$.
\end{defn}

\noindent
Note that the PseudoUCB Index uses a confidence bound, $\B\sqrt{\frac{2 \log t}{n_k(t)}}$, which is typically used in the UCB1 algorithm (\cite{auer2002finite}) for the objective of cumulative reward maximization. It has the property that $\Pr(I_\ell(t) < \mu_\ell) \leq Kt^{-3}$ [See \cref{lem:ucbindexmore1}], i.e., the probability of mean lying outside the pseudoUCB index $I_{\ell}(t)$ at round $t$ decays exponentially with the number of rounds $t$. This property allows us to show desirable sample complexity results for our proposed algorithm in \Cref{sec:regret}. We now present the C-LUCB algorithm, that makes use of the PseudoUCB, CrossUCB and LCB indices in its strategy for sampling arms, eliminating arms and stopping the algorithm.

\subsection{C-LUCB Algorithm}
The C-LUCB algorithm maintains a set of active arms $\mathcal{A}_t$, which is initialized to the set of all arms $\mathcal{K}=\{1, \dots, K\}$. At each round $t$, it samples arms, eliminates arms and then decides whether to stop as described below. 

\begin{enumerate}
\item \textbf{Sampling Strategy:} At each round $t$, the C-LUCB algorithm samples two arms $m_1(t)$ and $m_2(t)$, where $$m_1(t) = \argmax_{k \in \mathcal{A}_t} I_{k}(t), \quad m_2(t) = \argmax_{k \in \mathcal{A}_t \setminus \{m_1(t)\}} \min\left(\tilde{U}_{k,k}\left(t, \frac{\delta}{2K}\right), I_k(t)\right).$$

\item \textbf{Elimination Criteria:} The C-LUCB algorithm removes an arm $k$ from the set $\mathcal{A}_t$, if the CrossUCB index of arm $k$ is smaller than the LCB index of some other arm in $\mathcal{A}_t$, i.e., if $$\tilde{U}_k\left(t, \frac{\delta}{2K}\right) < \max_{\ell \in \mathcal{A}_t} L_\ell\left(t, \frac{\delta}{2K}\right).$$
Here, $\tilde{U}_\ell\left(t, \frac{\delta}{2K}\right) = \min_{k} \tilde{U}_{\ell,k}\left(t, \frac{\delta}{2K}\right).$

\item \textbf{Stopping Criteria:} If $|\mathcal{A}_t| = 1$, stop the algorithm and declare the arm in $\mathcal{A}_t$ as the optimal arm with $1 - \delta$ confidence.
\end{enumerate}

Both LUCB and C-LUCB sample the top two arms at round $t$ in $m_1(t)$ and $m_2(t)$ so as to resolve the ambiguity among them as fast as possible. However, C-LUCB uses the additional pseudo-reward information to modify its choice of $m_1(t)$ and $m_2(t)$. In particular, the use of $I_k(t)$ in definition of $m_2(t)$ avoids the sampling of an arm that appears sub-optimal from samples of other arms. Similarly, using the CrossUCB index $\tilde{U}_k\left(t, \delta/2K \right)$ instead of $\tilde{U}_{k,k}(t, \delta/2K)$, allows the C-LUCB to eliminate some arms earlier than the LUCB algorithm. A comparison of the operation of C-LUCB with LUCB and Racing based algorithms is presented in \Cref{tab:summaryAlgoClassic}. We show that the proposed C-LUCB algorithm is $1-\delta$ correct and analyze its sample complexity in the next section. As the key difference between C-LUCB and LUCB is in its sampling strategy, we explore some other variants of C-LUCB in \Cref{sec:variants}, where we study the effect of performance on altering the definitions of $m_1(t)$ and $m_2(t)$.

\section{Sample Complexity Results}
\label{sec:regret}
In this section, we analyze sample complexity of the proposed C-LUCB algorithm, that is, the number of samples required to identify the best arm with probability $1-\delta$. We show that some arms, referred to as \emph{non-competitive} arms, are explored implicitly through the samples of the optimal arm $k^*$ and contribute only an $\OO(1)$ term in the sample complexity, while other arms called \emph{competitive} arms have an $\OO\left(\log(1/\delta)\right)$ contribution in the sample complexity of the C-LUCB algorithm. The correlation information enables us to identify the non-competitive arms using samples from other arms and eliminate them early. For the sample complexity analysis, we assume that the rewards are bounded between $[0,1] \forall{k \in \mathcal{K}}$. Note that the algorithms do not require this condition and the analysis can also be generalized to any bounded rewards.

\subsection{Competitive and Non-competitive arms}
\label{sec:competitive}

We now define the notion of \emph{competitive} and \emph{non-competitive} arms, which are important to interpret our sample complexity results for the C-LUCB algorithm. Let $k^*$ denote the arm with the largest mean and $k^{(2)}$ denote the arm with the second largest mean. 

\begin{defn}[Non-Competitive and Competitive arms]
An arm $\ell$ is said to be non-competitive if the expected reward of the second best arm $k^{(2)}$ is strictly larger than the expected pseudo-reward of arm $\ell$ with respect to the optimal arm $\arm^*$, i.e,  $\optimistGap_{\ell} \triangleq (\mu_{k^{(2)}} - \phi_{\ell,k^*}) > 0$. Similarly, an arm $\ell$ is said to be competitive if $\optimistGap_{\ell} = (\mu_{k^{(2)}} - \phi_{\ell, k^*} ) \leq 0$. We refer to $\optimistGap_{\ell}$ as the pseudo-gap of arm $\ell$ in the rest of the paper. We denote the set of the competitive arms as $\mathcal{C}$ and the total number of competitive arms as $C$ in this paper.
\end{defn}

The best arm $k^*$ and second best arm $k^{(2)}$ have pseudo-gaps $\optimistGap_{k^*} = (\mu_{k^{(2)}} - \phi_{k^*,k^*}) < 0$ and $\optimistGap_{k^{(2)}} = (\mu_{k^{(2)}} - \phi_{k^{(2)}, k^*}) \leq 0$ respectively, and hence are counted in the set of competitive arms. As $\phi_{\ell,k^*} \geq \mu_\ell$, the pseudo-gap $\tilde{\Delta}_\ell \leq \Delta_\ell$. Due to this, we have $2 \leq C \leq K$.

The central idea behind our C-LUCB approach is that after sampling the optimal arm $k^*$ sufficiently large number of times, the non-competitive (and thus sub-optimal) arms will not be selected as $m_1(t)$ or $m_2(t)$ by the C-LUCB algorithm, and thus will not be explored explicitly. Furthermore, the non-competitive arms can be eliminated from the information obtained through arm $k^*$. As a result, the non-competitive arms contribute only an $\OO(1)$ term in the sample complexity, i.e., the contribution is independent of the confidence parameter $\delta$. However, the competitive arms cannot be discerned as sub-optimal by just using the rewards observed from the optimal arm, and have to be explored $\OO\left(\log\left(\frac{1}{\delta}\right)\right)$ times each. Thus, we are able to reduce a $K$-armed bandit to a $C$-armed bandit problem, where $C$ is the number of competitive arms. \footnote{Observe that $k^*$ and subsequently $C$ are both unknown to the algorithm. Before the start of the algorithm, it is not known which arm is optimal/competitive/non-competitive.}

\subsection{Analysis of C-LUCB}
We start by first proving the $(1- \delta)$-correctness of C-LUCB algorithm and then analyzing its sample complexity in terms of the number of samples obtained until the stopping criterion is satisfied.

\begin{thm}[$(1-\delta)$ correctness of C-LUCB]
Upon stopping, the C-LUCB algorithm declares arm $k^*$ as the best arm with probability $1 - \delta$.
\label{thm:correctness} 
\end{thm}
\noindent
\textit{Proof Sketch.}  To prove theorem 1, we define three events $\mathcal{E}_1, \mathcal{E}_2$ and $\mathcal{E}_3$ below.  Let $\mathcal{E}_1$ be the event that empirical mean of all arm lie within their confidence intervals uniformly for all $t \geq 1$
\begin{equation}
\mathcal{E}_1 = \Bigg\{ \forall{t \geq 1}, \forall{k \in \mathcal{K}}, ~~~ \hat{\mu}_k(t) - \Confidence\left(n_k(t), \frac{\delta}{2K}\right) \leq \mu_k \leq \hat{\mu}_k + \Confidence\left(n_k(t), \frac{\delta}{2K}\right) \Bigg\}    
\end{equation} 

\noindent
Define $\mathcal{E}_2$ to be the event that empirical pseudo-reward of optimal arm with respect to all other arms lie within their CrossUCB indices uniformly for all $t \geq 1$, i.e., 
\begin{equation}
\mathcal{E}_2 = \Bigg\{ \forall{t \geq 1}, \forall{\ell \in \mathcal{K}}, ~~~ \phi_{k^*,\ell} \leq \hat{\phi}_{k^*,\ell}(t) + \Confidence\left(n_\ell(t), \frac{\delta}{2K}\right) \Bigg\}    
\end{equation} 

Similarly define $\mathcal{E}_3$ to be the event that the empirical pseudo-reward of the sub-optimal arms with respect to the optimal arm lies within their CrossUCB indices uniformly for all $t \geq 1$, i.e., 
\begin{equation}
\mathcal{E}_3 = \Bigg\{ \forall{t \geq 1}, \forall{\ell \in \mathcal{K}}, ~~~ \phi_{\ell, k^*} \leq \hat{\phi}_{\ell, k^*}(t) + \Confidence\left(n_{k^*}(t), \frac{\delta}{2K}\right) \Bigg\}    
\end{equation} 
\noindent
Furthermore, we define $\mathcal{E}$ to be the intersection of the three events, i.e., 
\begin{equation}
\mathcal{E} = \mathcal{E}_1 \cap \mathcal{E}_2 \cap \mathcal{E}_3.
\label{eq:defE}
\end{equation}

\noindent
Due to the nature of anytime confidence intervals (See \Cref{eq:anytime}) and union bound over the set of arms, we have $\Pr(\mathcal{E}^c_1) \leq \frac{\delta}{2}$, $\Pr(\mathcal{E}_2^c) \leq \frac{\delta}{4}$ and $\Pr(\mathcal{E}^c_3) \leq \frac{\delta}{4}$ giving us $\Pr(\mathcal{E}^{c}) \leq \delta$. Furthermore, we show that, when event $\mathcal{E}$ occurs, the C-LUCB algorithm always declares $k^*$ as the best arm. This gives us the desired result in \Cref{thm:correctness}. A detailed proof is given in the \Cref{subsec:proofthm1}.

\begin{thm}
Given event $\mathcal{E}$ (defined in \Cref{eq:defE}), the expected number of samples drawn by C-LUCB until stopping, is bounded as
\begin{align}
 \E{N^{\text{C-LUCB}} \mid \mathcal{E}} \leq \sum_{k \in \mathcal{C}} \frac{2\constant}{\Delta_k^2} \log \left(\frac{2K \log\left(\frac{1}{\Delta_k^2}\right)}{\delta}\right) +  \frac{3K + 2Kt_0}{1 - \delta} +  \frac{2}{1 - \delta}\left(\frac{(K+1)^3}{t_0} + \frac{2}{t_0^2}\right),   
\end{align}

where $\slot_0  = \inf \bigg\{\tau \geq 2: \Delta_{k^*}  \geq 4 \sqrt{\frac{2K\log \tau}{\tau}} ~ \forall{k \notin \mathcal{C}} \bigg\}$ and $\constant$ is a universal constant that depends on the type of confidence bound used to construct $\Confidence(n_k, \delta)$ (Section 3b) -- the tighter the bound, the smaller the $\constant$. The gap $\Delta_{k}$ is defined as $\Delta_k \triangleq \mu_{k^*} - \mu_k \quad \text{for } k \neq k^*$, i.e., the difference in mean reward of optimal arm $k^*$ and mean reward of arm $k$ and $\Delta_{k^*} \triangleq \min_{k \neq k^*} \Delta_k$, i.e., the gap between best and second best arm.
\label{thm:sampComp}
\end{thm}

We present a brief proof outline below, while the detailed proof is available in the \Cref{subsec:proofthm2}. 

\noindent
\textit{Proof Sketch.} In order to bound the total number of samples drawn by C-LUCB, we bound the total number of rounds $T$ taken by C-LUCB before stopping. As C-LUCB algorithm pulls two arms $m_1(t)$ and $m_2(t)$ in each round $t$, the number of samples $N^{\text{C-LUCB}} = 2T$. We obtain an upper bound on the total number of rounds $T$, considering the following four counts of the number of rounds and obtain an upper bound for each of them under the event $\mathcal{E}$:
\begin{enumerate}
    \item \emph{$T^{(\mathcal{R})}$: } Let $T^{(\mathcal{R})}$ denote the number of rounds in which $I_{k^*}(t) < \mu_{k^*}$, i.e., the count of events in which the pseudoUCB index of arm $k^*$ is smaller than the mean of arm $k^*$ at round $t$.
    \item \emph{$T^{(C)}$: } Define $T^{(C)}$ to be the number of rounds in which $m_1(t), m_2(t) \in \mathcal{C}$ and event $I_{k^*}(t) < \mu_{k^*}$ does not occur. 
    \item \emph{$T^{(NC)}$: } Define $T^{(NC)}$ to be the number of rounds in which $m_1(t) \notin \mathcal{C}, m_2(t) \neq k^*$ or $m_2(t) \notin \mathcal{C}, m_1(t) \neq k^*$. 
    \item \emph{$T^{(*)}$:} Define $T^{(*)}$ to be the number of rounds in which $m_1(t) = k^*, m_2(t) \notin \mathcal{C}$ or $m_2(t) = k^*, m_1(t) \notin \mathcal{C}$ .
\end{enumerate}
We can now see that $T \leq T^{(\mathcal{R})} + T^{(C)} + T^{(NC)} + T^{(*)}$. We show that \newadd{$$\Pr(I_{k^*}(t) < \mu_{k^*}|\mathcal{E}) = \frac{\Pr(I_{k^*} < \mu_{k^*}, \mathcal{E})}{\Pr(\mathcal{E})} \leq \frac{\Pr(I_{k^*} < \mu_{k^*}, \mathcal{E})}{1 - \delta} \leq \frac{\Pr(I_{k^*} < \mu_{k^*})}{1 - \delta} \leq \frac{Kt^{-3}}{1 - \delta},$$} giving us $\E{T^{(\mathcal{R})}|\mathcal{E}} \leq \frac{1}{1 - \delta}\sum_{t = 1}^{\infty} Kt^{-3} \leq \frac{3K}{2(1-\delta)}$. Next we show that \\
$\Pr\left(T^{(C)} + T^{(*)} \geq \sum_{k \in \mathcal{C}} \frac{\constant}{\Delta_k^2} \log \left(\frac{2K \log\left(\frac{1}{\Delta_k^2}\right)}{\delta}\right) \Big|  \mathcal{E} \right) = 0$. Due to this, $$T^{(C)} + T^{(*)} \leq \sum_{k \in \mathcal{C}} \frac{\constant}{\Delta_k^2} \log \left(\frac{2K \log\left(\frac{1}{\Delta_k^2}\right)}{\delta}\right) \quad \text{w.p. } 1-\delta.$$ We then evaluate an upper bound on $\E{T^{(NC)}|\mathcal{E}}$ and show that it is upper bounded by a $\OO(1)$ constant, i.e., $$\E{T^{(NC)}|\mathcal{E}} \leq \frac{Kt_0}{1 - \delta} +  \frac{1}{1 - \delta}\left(\frac{(K+1)^3}{t_0} + \frac{2}{t_0^2}\right).$$

Putting these results together, we obtain the result of \Cref{thm:sampComp}.

Furthermore, as $\E{T^{(NC)} |\mathcal{E}}, \E{T^{(\mathcal{R})} |\mathcal{E}}$ is upper bounded by an $\OO(1)$ constant as $\delta \rightarrow 0$, \newadd{we have $\sum_{t=1}^{\infty} \Pr(\mathcal{E}^{\text{NC}}_t) < \infty$, where $\mathcal{E}^{\text{NC}}_t$ is the event that $m_1(t) \notin \mathcal{C}, m_2(t) \neq k^*$ or $m_2(t) \notin \mathcal{C}, m_1(t) \neq k^*$.  By Borel-Cantelli Lemma 1, this implies that  with probability 1, the event $\mathcal{E}^{\text{NC}}_t$ takes place only finitely many time  steps $t$.} As a result of this, $\exists d_1  : \Pr(T^{(NC)} > d_1 | \mathcal{E}) = 0 $ \emph{almost surely}. Similarly $\exists d_2: \Pr(T^{(\mathcal{R})} > d_2 | \mathcal{E}) = 0$ $a.s.$ As a consequence of this, we have the following result bounding the total number of samples drawn from the C-LUCB algorithm with probability $1 - \delta$.

\begin{coro}
The number of samples obtained by C-LUCB is upper bounded as
\begin{equation}
    N^{\text{C-LUCB}} \leq \sum_{k \in \mathcal{C}} \frac{2\constant}{\Delta_k^2} \log \left(\frac{2K \log\left(\frac{1}{\Delta_k^2}\right)}{\delta}\right) + d \quad \text{w.p. } 1 - \delta,
\end{equation}
\label{coro:mycoro}
\end{coro}
\noindent
where $d = \max(d_1, d_2)$. Note that the $\OO\left(\log\left(\frac{1}{\delta}\right)\right)$ term is only summed for the set of competitive arms $\mathcal{C}$, in contrast to the LUCB algorithm where the sample complexity term involves summation of a $\OO\left(\log\left(\frac{1}{\delta}\right)\right)$ for all arms $k \in \mathcal{K}$. In this sense, our proposed algorithm reduces a $K$-armed bandit problem to a $C$-armed bandit problem.

The key intuition behind our sample complexity result is that the sampling of $m_1(t) = \argmax_{k \in \mathcal{A}_t} I_k(t)$ ensures that the optimal arm is sampled at least $t/K$ times till round $t$ with high-probability. This in turn ensures that the non-competitive arms are not selected as $m_1(t)$ or $m_2(t)$, due to which we see that their expected number of samples are bounded above by a $\OO(1)$ constant.

\subsection{Comparison with the LUCB algorithm}

The LUCB algorithm is known to stop after obtaining at most $\left(\sum_{k \in \mathcal{K}} \frac{2\constant}{\Delta_k^2} \log \left( \frac{K \log\left(\frac{1}{\Delta_k^2}\right)}{\delta}\right)\right)$ samples with probability at least $1-\delta$. More formally, $$N^{\text{LUCB}} \leq  \left(\sum_{k \in \mathcal{K}} \frac{2\constant}{\Delta_k^2} \log \left( \frac{K \log\left(\frac{1}{\Delta_k^2}\right)}{\delta}\right)\right), \quad \text{ w.p. } 1 - \delta.$$  We compare this result with the one that we prove for C-LUCB algorithm in \Cref{thm:sampComp}.

\noindent
\textbf{Reduction to a $C$-Armed Bandit problem: } As highlighted earlier, in the C-LUCB approach, the $\OO\left(\log\left(\frac{1}{\delta}\right)\right)$ term only comes from the set of competitive arms, as opposed to the LUCB algorithm which has $\OO(\left(\log\left(\frac{1}{\delta}\right)\right)$ contribution from all its arms. In this sense, C-LUCB algorithm reduces a $K$-armed bandit problem to a C-armed bandit problem. Depending on the problem instance, the value of $C$ can vary between $2$ and $K$.

\noindent
\textbf{Slightly larger number of samples from competitive arms: } We see that the contribution coming from a competitive arm in C-LUCB algorithm is $\frac{2\constant}{\Delta_k^2} \log \left( \frac{2K \log\left(\frac{1}{\Delta_k^2}\right)}{\delta}\right)$. This is slightly larger than the contribution coming from a sub-optimal arm in LUCB algorithm, where each arm contributes $\frac{2\constant}{\Delta_k^2} \log \left( \frac{K \log\left(\frac{1}{\Delta_k^2}\right)}{\delta}\right)$ in the sample complexity. This is due to the fact that we construct slightly wider confidence intervals, $\Confidence\left(n_k, \frac{\delta}{2K}\right)$ instead of $\Confidence\left(n_k, \frac{\delta}{K}\right)$, in C-LUCB to take advantage of the correlations present in the problem. We see in \Cref{sec:experiments} that this small increase in the width of confidence intervals does not have a significant impact on the empirical performance of the algorithm. 

\noindent
\textbf{\Cref{thm:sampComp}'s result is in conditional expectation: } While the sample complexity result of the LUCB algorithm bounds the total number of samples taken with probability $1 - \delta$, our sample complexity result bounds the expected samples taken by C-LUCB algorithm under the event $\mathcal{E}$ (\Cref{thm:sampComp}). This arises as the analysis of our algorithm requires a transient component, because it tries to avoid sampling non-competitive arm at each round with high probability. We have a result in \Cref{coro:mycoro} that evaluates an upper bound which holds with probability $1 - \delta$, but we are unable to quantify the constant $d$ in \Cref{coro:mycoro} and can only characterize $d$ in expectation as done in \Cref{thm:sampComp}. \newadd{An open problem is to evaluate the expected sample complexity of our C-LUCB algorithm for the cases where the event $\mathcal{E}$ does not occur. While such results are hard to obtain theoretically, in all our experiments we observed that the variance in the number of samples drawn by C-LUCB is not much, and is in fact similar to that of the LUCB algorithm in all the experiments performed. This indicates that even when algorithm stops with an incorrect arm, the number of samples obtained are similar to the samples obtained under the good event $\mathcal{E}$. }

\noindent
\textbf{The $\log(K)$ term in numerator: } Just like the sample complexity result of the LUCB algorithm \cite{jamieson2014best}, our sample complexity result also has a $\log(K)$ in its sample complexity result. This is avoidable in the classical MAB framework if one uses the lil'UCB algorithm, which is known to have the optimal theoretical sample complexity in the classical bandit setting as it avoids the $\log(K)$ term in its sample complexity expression. However the use of lil'UCB algorithm leads to worse empirical performance as seen in our experiments and prior work \cite{jamieson2014best}. Due to this reason, we focus only on the extension of LUCB to the correlated bandit setting. \newadd{The LUCB++ algorithm has a sample complexity of the form of $\left(\sum_{k \in \mathcal{K} \setminus \{k^*\}} \frac{2\constant_1}{\Delta_k^2} \log \left( \frac{ \log\left(\frac{1}{\Delta_k^2}\right)}{\delta}\right) + \frac{2\constant_2}{\Delta_{k^*}^2} \log \left( \frac{K \log\left(\frac{1}{\Delta_{k^*}^2}\right)}{\delta}\right)\right)$. The LUCB++ algorithm avoids the $\log(K)$ term in the sample complexity for the sub-optimal arms and has it only for the optimal arm $k^*$. Due to this, it is seen that LUCB++ slightly outperforms the LUCB algorithm empirically. In our next section, we propose the C-LUCB++ algorithm, which is a heuristic extension of LUCB++ to the correlated bandit setting and show that it finds the optimal arm with probability at least $1-\delta$. } 

\noindent
\textbf{Dependency with $K$: } \newadd{In our sample complexity results, the dependence with respect to $K$ is loose. For our theoretical results, we focus on studying the dependence of sample complexity on $\delta$ in this paper. In \Cref{sec:experiments}, we show that even when $\delta = 0.1$ (i.e., a moderate confidence regime), our proposed algorithms outperform the classical bandit algorithms (See \Cref{fig:introFig}).}

\section{Variants of C-LUCB}
\label{sec:variants}

\begin{table}[t]
\resizebox{\textwidth}{!}{
\centering
\begin{tabular}{l l l l }
\hline
\textbf{Algorithm} & \textbf{First arm $m_1(t)$} & \textbf{Second arm $m_2(t)$} & Samples drawn \\ \hline
\textbf{C-LUCB} &  $\argmax\limits_{k \in \mathcal{A}_t} I_k(t)$   & $\argmax\limits_{k \in \mathcal{A}_t \setminus \{ {m_1}\}} \min\left(\tilde{U}_{k,k}\left(\frac{\delta}{2K}\right), I_k(t)\right)$  &  $39277.8$     \\ \hline
\textbf{maxmin-LUCB} & $\argmax\limits_{k \in \mathcal{A}_t} \min_{\ell} \hat{\phi}_{k,\ell}(t)$   & $\argmax\limits_{k \in \mathcal{A}_t \setminus \{ {m_1}\}} \tilde{U}_k\left(\frac{\delta}{2K}\right)$    &  $36314.2$   \\ \hline
\textbf{2-LUCB} & $\argmax\limits_{k \in \mathcal{A}_t} I_k(t)$   & $\argmax\limits_{k \in \mathcal{A}_t \setminus \{ {m_1}\}} \tilde{U}_k\left(\frac{\delta}{2K}\right)$    &   $39385.8$   \\ \hline 
\end{tabular}

}
\caption{We study two intuitive variants of C-LUCB which differ in their sampling strategy of $m_1(t)$ and $m_2(t)$. Both of them have same elimination and stopping criteria as the C-LUCB algorithm. We report the number of samples needed to identify the best genre from the set of 18 movie genres in the Movielens dataset. While all of these are smaller than the samples drawn by LUCB (which is 61175.4 in this case), the difference between the variants of C-LUCB is minimal. Experimental details are described in detail in \Cref{sec:experiments}, we set the value of $p = 0.2$ (i.e., the fraction of pseudo-reward entries that are replaced by $5$) in this experiment. Such similarity in empirical performance has also been observed in our other experiments and we found no clear winner among the three when compared on their empirical performance. }
\label{tab:variantsCLUCB}
\vspace{-0.2cm}
\end{table}

In our proposed C-LUCB algorithm, at each round we sample two arms $m_1(t), m_2(t)$, where $m_1(t) = \argmax_{k \in \mathcal{A}_t} I_k(t) $ and $m_2(t) = \argmax_{k \in \mathcal{A}_t \setminus \{ {m_1}\}} \min(\tilde{U}_{k,k}(\delta/2K), I_k(t))$. A sampling such as this allowed us to show $1 - \delta$ correctness of the algorithm (\Cref{thm:correctness}) and analyse its sample complexity (\Cref{thm:sampComp}). In this section, we explore two other algorithms, that we call maxmin-LUCB and 2-LUCB, that sample different $m_1(t)$ and $m_2(t)$ at round $t$, but have the same elimination and stopping criteria as that of C-LUCB. In \Cref{tab:variantsCLUCB}, we contrast their sampling strategy with respect to C-LUCB. While we are able to show that both maxmin-LUCB and 2-LUCB algorithm will stop with the best-arm with probability at least $1-\delta$, we are unable to provide a sample complexity result for them. 

We also evaluated the empirical performance of maxmin-LUCB and 2-LUCB on a real-world recommendation dataset, and found their empirical performance to be similar to C-LUCB. We chose to use C-LUCB as our proposed algorithm as it is possible to provide theoretical guarantees as in \Cref{thm:correctness} and \Cref{thm:sampComp}. Moreover, we find its empirical performance to be superior than classical bandit algorithms in correlated bandit settings, as we illustrate through our experiments in the next section.

\subsection{C-LUCB++: Heuristic extension of LUCB++}

The LUCB++ algorithm as illustrated in \Cref{sec:priorWork}, is able to improve upon LUCB, by modifying its stopping criteria and in its sampling of $m_1(t)$ and $m_2(t)$. We propose an extension, C-LUCB++, that extends the LUCB++ algorithm to the correlated bandit setting. The comparison of C-LUCB++ and LUCB++ in its sampling, elimination and stopping criteria is presented in \Cref{tab:summaryAlgoClassic}. While we are able to show that the C-LUCB++ stops with the best arm with probability at least $1 - \delta$ in \Cref{subsec:plusplus}, analysing its sample complexity remains an open problem. We compare the performance of C-LUCB++, with C-LUCB, LUCB, Racing and lil'UCB algorithms extensively through our experiments on Movielens and Goodreads datasets in the next section.

\section{Experiments}
\label{sec:experiments}

We now evaluate the performance of our proposed C-LUCB and C-LUCB++ algorithms in a real-world setting. By comparing the performance against classical best-arm identification algorithms on the \textsc{movielens} and \textsc{goodreads} datasets, we show that our proposed algorithms are able to exploit correlation to identify the best-arm in fewer samples. All results reported in our paper are presented after conducting 10 independent trials and computing their average. Additionally, in all our plots we show the error bars of width 2$\sigma$, where $\sigma$ is the standard deviation in the number of samples drawn by an algorithm across the 10 independent trials.

\begin{figure}[t]
    \centering
    \includegraphics[width = 0.9\textwidth]{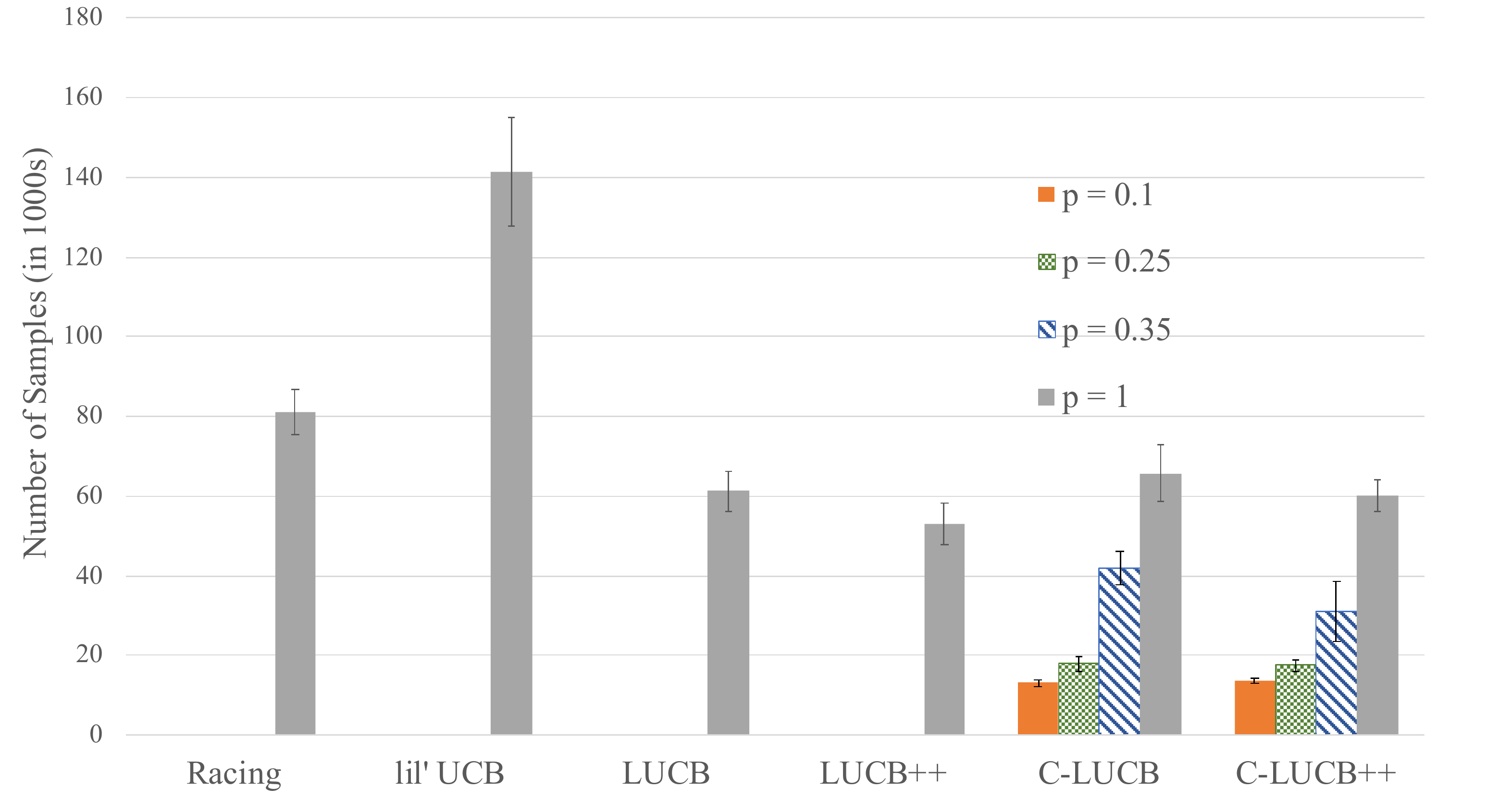}
    \caption{Number of samples drawn by Racing, lil'UCB, LUCB, LUCB++, C-LUCB and C-LUCB++ to identify the best movie genre out of 18 possible genres in the Movielens dataset. Here, $p$ represents the fraction of pseudo-reward entries that are replaced by the maximum possible reward (i.e., 5). When $p$ is small, there is more correlation information available that our proposed C-LUCB and C-LUCB++ algorithms exploit to reduce the number of samples needed to identify the best movie genre. When $p = 1$, there is no correlation information available, in which case our proposed C-LUCB and C-LUCB++ algorithms have a performance similar to LUCB and LUCB++ respectively.
    }
    \label{fig:movielensGenre}
\end{figure}

\subsection{Experiments on the \textsc{MovieLens} dataset}
The \textsc{MovieLens} dataset \cite{movielenspaper} contains a total of 1M ratings for a total of 3883 Movies rated by 6040 Users. Each movie is rated on a scale of 1-5 by the users. Moreover, each movie is associated with one (and in some cases, multiple) genres. For our experiments, of the possibly several genres associated with each movie, one is picked uniformly at random. To perform our experiments, we split the data into two parts, with the first half containing ratings of the users who  provided the most number of ratings. This half is used to learn the pseudo-reward entries, the other half is the test set which is used to evaluate the performance of the proposed algorithms. Doing such a split ensures that the rating distribution is different in the training and test data.

\noindent
\textbf{Best Genre identification.} In this experiment, our goal is to identify the most preferred genre among the 18 different genre in the test population in fewest possible samples. The pseudo-reward entry $s_{\ell,k}(r)$ is evaluated by taking the empirical average of the ratings of genre $\ell$ that are rated by the users who rated genre $k$ as $r$. As in practice, all such pseudo-reward entries might not be available, we randomly replace $p$-fraction of the pseudo-reward entries by maximum possible reward, i.e., $5$. We then run our best-arm identification algorithms on the test data to identify the best-arm with $99\%$ confidence.  \Cref{fig:movielensGenre} shows the average samples taken by C-LUCB and C-LUCB++ algorithm relative to the classical best-arm identification algorithms for different value of $p$ (the fraction of pseudo-reward entries that are removed). We see that C-LUCB and C-LUCB++ algorithms significantly outperform all Racing, lil'UCB, LUCB and LUCB++ algorithms for $p = 0.1, 0.25, 0.35$ as they are able to exploit the correlations present in the problem to identify the best arm in a faster manner.

In the scenario where all pseudo-reward entries are unknown, i.e., $p = 1$, we see that the performance of C-LUCB is only slightly worse than that of LUCB algorithm. This is due to the construction of slightly wide confidence interval $\Confidence(n_k, \delta/2K)$ for the C-LUCB algorithm relative to LUCB algorithm that uses $\Confidence(n_k, \delta/K)$. We also see that in this scenario, LUCB++ and C-LUCB++ algorithm (which is an extension of LUCB++) outperform C-LUCB, which is due to the known superiority of LUCB++ over LUCB \cite{simchowitz2017simulator, tanczos2017kl}.

\noindent
\textbf{Variation with $\delta$.} We then study the performance of the best-arm identification algorithms for different value of $\delta$. In \Cref{fig:introFig}, we plot the number of samples required by C-LUCB and C-LUCB++ to identify the best arm with $90\%, 94\%, 98\%$ and $99\%$ confidence, with $p = 0.2$ (i.e., $20\%$ of pseudo-reward entries are replaced by 5). As C-LUCB and C-LUCB++ are able to make use of the available correlation information, we see our proposed algorithms require fewer samples than the Racing, lil'UCB, LUCB and LUCB++ algorithms in each of the four settings.

\begin{figure}[t]
    \centering
    \includegraphics[width = 0.9\textwidth]{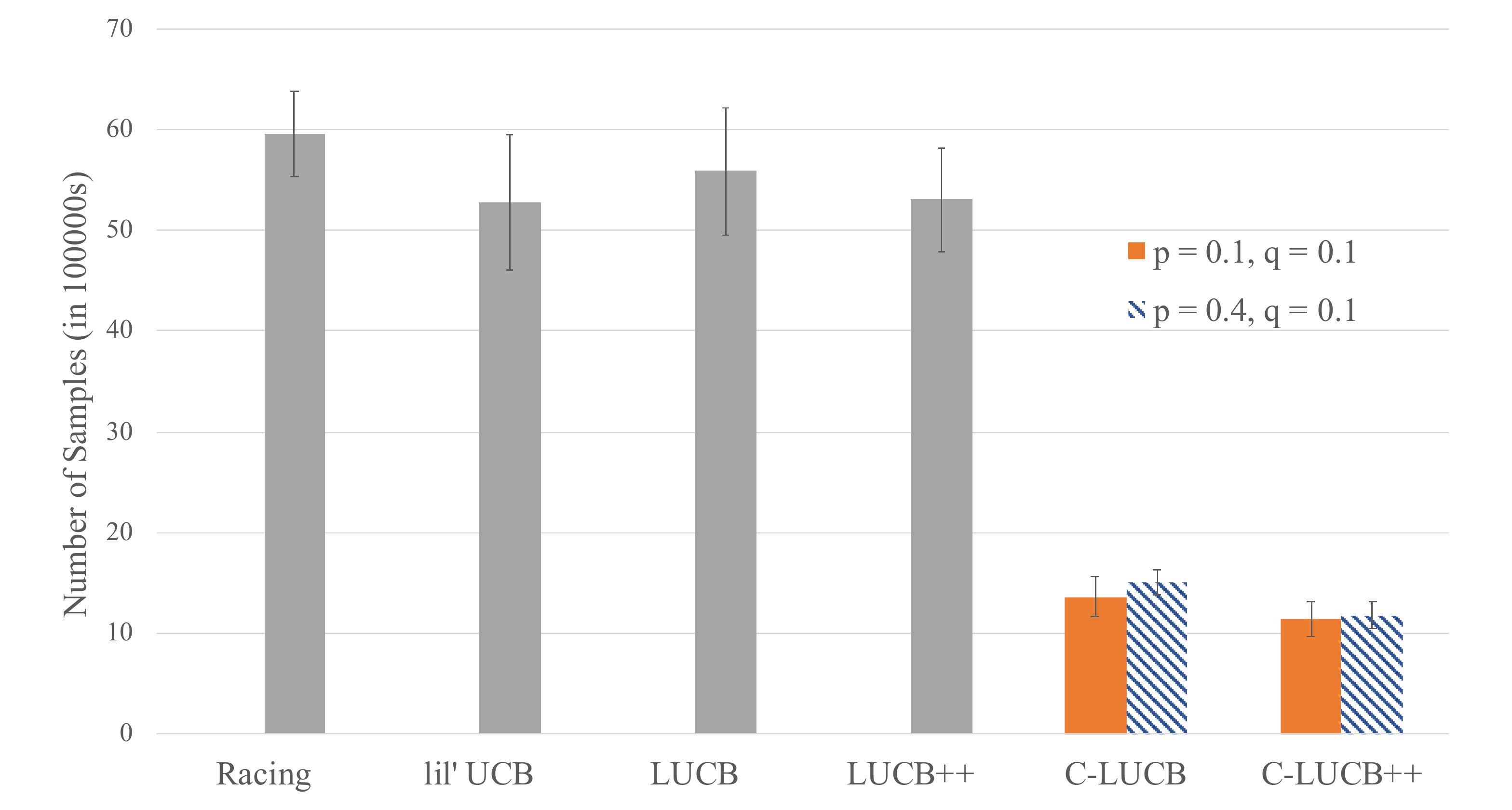}
    \caption{Number of samples needed by Racing, lil'UCB, LUCB, LUCB++, C-LUCB and C-LUCB++ to identify the best poem out of the set of 25 poem books in the Goodreads dataset. Here $p$ represents the fraction of pseudo-rewards that are replaced by maximum possible reward and $q = 0.1$ is added to each pseudo-reward entry to account for the fact that pseudo-reward entries may be noisy. Our proposed C-LUCB and C-LUCB++ utilize correlation information and require significantly less samples than the classical best-arm identification algorithms. 
    }
    \label{fig:goodreads}
\end{figure}

\subsection{Experiments on the {\sc Goodreads} dataset}
The {\sc Goodreads} dataset \cite{wan2018item} contains the ratings for 1,561,465 books by a total of 808,749 users. Each rating is on a scale of 1-5. For our experiments, we only consider the poetry section and focus on the goal of identify the most liked poem for the population. The poetry dataset has 36,182 different poems rated by 267,821 different users. We do the pre-processing of goodreads dataset in the same manner as that of the MovieLens dataset, by splitting the dataset into two halves, train and test. The train dataset contains the ratings of the users with most number of recommendations.

\noindent
\textbf{Best book identification.} We consider the 25 most rated poetry books in the dataset and aim to identify the best book in fewest possible samples with $99\%$ confidence. After obtaining the pseudo-reward entries from the training data, we replace $p$ fraction of the entries with the highest possible reward (i.e., $5$) as some pseudo-rewards may be unknown in practice. To account for the fact that these pseudo-reward entries may be noisy in practice, we add a safety buffer of $0.1$ to each of the pseudo-reward entry $s_{\ell,k}(r)$; i.e., we set the pseudo-reward to be empirical conditional mean (obtained from training data) {\em plus} the safety buffer $q = 0.1$. We perform experiment on the test data and compare the number of samples obtained for different algorithms in \Cref{fig:goodreads} for two different values of $p$. We see that in both the cases, our C-LUCB and C-LUCB++ algorithms outperform other algorithms as they are able to exploit the correlations in the rewards.

\section{Concluding Remarks}
In this work, we studied a new multi-armed bandit problem, where rewards corresponding to different arms are correlated to each other and this correlation is known and modeled through the knowledge of pseudo-rewards. These pseudo-rewards are \emph{loose} upper bounds on conditional expected rewards and can be evaluated in practical scenarios through controlled surveys or from domain expertise. We then extended an LUCB based approach to perform best-arm identification in the correlated bandit setting. Our approach makes use of the pseudo-rewards to reduce the number of samples taken before stopping. In particular, our approach avoids the sampling of non-competitive arms leading to a stark reduction in sample complexity. The theoretical superiority of our proposed approach is reflected in practical scenarios. Our experimental results on Movielens and Goodreads recommendation dataset show that the presence of correlation, when exploited by our C-LUCB approach, can lead to significant reduction in the number of samples required to identify the best-arm with probability $1 - \delta$.

This work opens up several interesting future directions, including but not limited to the following: \\
\noindent
\textbf{PAC-C-LUCB: } In this work, we explored the problem of identifying the best-arm with probability $1-\delta$. A closely related problem is to find a PAC (probably approximately correct) algorithm, that identifies an arm which is within $\epsilon$ from $\mu_{k^*}$ with probability at least $1 - \delta$. We believe such an algorithm can be constructed by modifying the elimination and stopping criteria of C-LUCB algorithm. More specifically, if one compares $U_{k}(n_k, \delta) + \epsilon $ v/s $\max_{k \in \mathcal{A}_t} L_k(n_k, \delta)$ in the C-LUCB's elimination criteria, it may be possible to design and analyse a PAC algorithm in the correlated multi-armed bandit setting. 

\noindent
\textbf{Using Pseudo-Lower bounds: } We assume in our work that only upper bounds on conditional expected rewards, in the form of pseudo-upper-bounds, are known to the player. In practical settings, it may also be possible to obtain pseudo-lower-bounds, that may allow us to know information about lower bound on conditional expected reward. In presence of such knowledge, we believe C-LUCB algorithm will need a modification in its definition of lower confidence bound $L_k(n_k, \delta)$. By defining a crossLCB index $L_{\ell, k}(n_k, \delta)$, equivalent to crossUCB index for upper bound, we can re-define $L_k = \max L_{\ell,k}$. This new definition of the lower confidence bound index can help us to incorporate cases where pseudo-lower bounds are also known.

\noindent
\textbf{Top $m$ arms identification: } Throughout this work, our focus was to identify just the optimal arm from the set of $K$ arms. Another similar problem is to come up with an approach to find the best $m$ arms from the set of $K$ arms. It is an interesting direction to explore in the correlated-multi armed bandit setting. We believe such a problem would be even more interesting if the pseudo-lower bounds are known. An open problem is to extend a C-LUCB like approach to identify the best $m$ arms from the set of $K$ arms.

\noindent
\textbf{Lower bound and optimal solution: } While our proposed approach shows promising empirical performance and has some theoretical guarantees, it may not be the optimal solution for the correlated bandit problem studied in this paper. Studying a lower bound and correspondingly an optimal solution to this problem remains an open problem.

\bibliographystyle{ieeetr}
\bibliography{multi_armed_bandit}

\newpage
\onecolumn
\appendix

\section{Standard Results from Previous Works}
\begin{fact}[Hoeffding's inequality]
Let $Z_1, Z_2 \ldots Z_n$ be i.i.d random variables bounded between $[a, b]: a \leq Z_i \leq b$, then for any $\delta>0$, we have
$$\Pr\left(\left| \frac{\sum_{i = 1}^{n} Z_i}{n} - \E{Z_i} \right| \geq \delta\right) \leq \exp \left( \frac{-2 n \delta^2}{(b - a)^2}\right).$$ 
\end{fact}

\begin{lem}[Standard result used in bandit literature]
If $\hat{\mu}_{k,n_k(t)}$ denotes the empirical mean of arm $k$ by sampling arm $k$ $n_k(t)$ times through any algorithm and $\mu_k$ denotes the mean reward of arm $k$, then we have 
$$\Pr\left(\hat{\mu}_{k,n_k(t)} - \mu_k \geq \epsilon, \tau_2 \geq n_k(t) \geq \tau_1 \right) \leq \sum_{s = \tau_1}^{\tau_2}\exp \left(- 2 s \epsilon^2\right).$$ 
\label{lem:UnionBoundTrickInt}
\end{lem}

\begin{proof}
Let $Z_1, Z_2, ... Z_t$ be the reward samples of arm $k$ drawn separately. If the algorithm chooses to sample arm $k$ for $m^{th}$ time, then it observes reward $Z_m$. Then the probability of observing the event $\hat{\mu}_{k,n_k(t)} - \mu_k \geq \epsilon, \tau_2 \geq n_k(t) \geq \tau_1$ can be upper bounded as follows,
\begin{align}
    \Pr\left(\hat{\mu}_{k,n_k(t)} - \mu_k \geq \epsilon, \tau_2 \geq n_k(t) \geq \tau_1 \right) &= \Pr\left( \left( \frac{\sum_{i=1}^{n_k(t)}Z_i}{n_k(t)} - \mu_k \geq \epsilon \right), \tau_2 \geq n_k(t) \geq \tau_1 \right) \\
    &\leq \Pr\left( \left(\bigcup_{m = \tau_1}^{\tau_2} \frac{\sum_{i=1}^{m}Z_i}{m} - \mu_k \geq \epsilon \right), \tau_2 \geq n_k(t) \geq \tau_1 \right) \label{upperBoundTrick}\\
    &\leq \Pr \left(\bigcup_{m = \tau_1}^{\tau_2} \frac{\sum_{i=1}^{m}Z_i}{m} - \mu_k \geq \epsilon \right) \\
    &\leq \sum_{s = \tau_1}^{\tau_2}\exp \left( - 2 s \epsilon^2\right).
\end{align}
\end{proof}

\begin{lem}[From Proof of Theorem 1 in \cite{auer2002finite}]
\label{lem:ucbindexmore}
The probability that the mean reward of arm $k$, i.e., $\mu_k$, is greater than the pseudoUCB index of arm $k$ with respect to arm $k$, i.e., $I_{k,k} = \hat{\mu}_k + \sqrt{\frac{2\log t}{n_k(t)}}$ is upper bounded by $t^{-3}$.
$$\Pr(\meanReward_\arm > \Index_{k,k}(\slot)) \leq \slot^{-3}.$$ 
\end{lem}
Observe that this bound does not depend on the number $\pulls_\arm(\slot)$ of times arm $\arm$ is sampled and only depends on $t$.
\begin{proof}
This proof follows directly from \cite{auer2002finite}. We present the proof here for completeness as we use this frequently in the paper.
\begin{align}
    \Pr(\meanReward_\arm > \Index_{\arm,\arm}(\slot)) &= \Pr\left(\meanReward_\arm > \hat{\meanReward}_{\arm,\pulls_\arm(\slot)} + \sqrt{\frac{2 \log \slot}{\pulls_\arm(\slot)}}\right) \\
    &\leq \sum_{m = 1}^{\slot} \Pr \left(\meanReward_\arm > \hat{\meanReward}_{\arm,m} + \sqrt{\frac{2 \log \slot}{m}} \right) \label{unionTrick}\\
    &= \sum_{m =1}^{\slot} \Pr \left(\hat{\meanReward}_{\arm,m} - \meanReward_\arm < - \sqrt{\frac{2 \log \slot}{m}}\right) \\ 
    &\leq \sum_{m = 1}^{\slot} \exp\left(- 2 m \frac{2 \log \slot}{m}\right) \label{eqn:ucbindex}\\
    &= \sum_{m = 1}^{\slot} \slot^{-4} \\
    &= \slot^{-3}.
\end{align}
where \eqref{unionTrick} follows from the union bound and is a standard approach (\Cref{lem:UnionBoundTrickInt}) to deal with random variable $\pulls_\arm(\slot)$. We use this approach repeatedly in the proofs. We have \eqref{eqn:ucbindex} from the Hoeffding's inequality. Note that if the empirical mean $\mu_k$ is replaced by the empirical pseudo reward of arm $k$ with respect to arm $\ell$, i.e., $\phi_{k,\ell}$ and $I_{k,k}(t)$ by the expected pseudo reward of arm $k$ with respect to arm $\ell$, i.e., $I_{k,\ell}(t) = \hat{\phi}_{k,\ell} + \sqrt{\frac{2\log t}{n_\ell(t)}}$. Then we get that $\Pr(\phi_{k,\ell} > I_{k,\ell}(t)) \leq t^{-3}$ using the same steps as presented above.
\end{proof}

\section{Intermediate lemmas for proving bounds on samples obtained through non-competitive arms}
\begin{lem}
\label{lem:ucbindexmore1}
Let $\Index_\arm(\slot)$ denote the pseudoUCB index of arm $\arm$ at round $\slot$, and $\meanReward_\arm$ denote the mean reward of that arm.  Then, we have
$$\Pr(\meanReward_{k} > \Index_{k}(\slot)) \leq K\slot^{-3}.$$ 
\end{lem}
Similar to \Cref{lem:ucbindexmore}, this bound does not depend on the number of times arm $k$ is sampled till round $t$ (i.e., $n_k(t)$) and only depends on the round $t$ and the total number of arms $K$. Recall that $I_\ell(t) = \min_k I_{\ell,k}(t)$, where $I_{\ell,k}(t)$ is PseudoUCB index of arm $\ell$ with respect to arm $k$ defined in \eqref{eq:pseudoUCBIndexdefn}.
\begin{proof}
This proof follows in the same way as that of \Cref{lem:ucbindexmore}. 
\begin{align}
    \Pr(\meanReward_{k} > \Index_{k}(\slot)) &= \Pr\left(\mu_{k} > \min_{\ell} \hat{\phi}_{k, \ell} + \sqrt{\frac{2 \log t}{n_{\ell}(t)}}\right) \label{eq:stepBydefn} \\  
    &\leq \sum_{\ell \in \mathcal{K}} \Pr\left(\mu_{k} > \hat{\phi}_{k, \ell} + \sqrt{\frac{2 \log t}{n_{\ell}(t)}}\right) \\
    &\leq \sum_{\ell \in \mathcal{K}} \Pr\left(\phi_{k, \ell} > \hat{\phi}_{k, \ell} + \sqrt{\frac{2 \log t}{n_{\ell}(t)}}\right) \label{eq:stepUB}\\
    &\leq \sum_{\ell \in \mathcal{K}} t^{-3} \label{eq:bdd}\\
    &= Kt^{-3}.
\end{align}
\end{proof}
We have \eqref{eq:stepBydefn} from the definition of $I_k(t)$. Inequality \eqref{eq:stepUB} follows from the fact that $\phi_{k, \ell} \geq \mu_k$. We get \eqref{eq:bdd} follows from the hoeffding's inequality combined with the union bound (\Cref{lem:ucbindexmore}).

\begin{lem}
If $k \neq k^*$ is a non-competitive arm i.e., $k \notin \mathcal{C}$ and has a pseudo-gap $\optimistGap_{\arm,\arm^*} > 0$, then,
$$\Pr((m_1(t) = \arm \cup m_2(t) = \arm), \pulls_{\arm^*}(\slot) \geq t/2K, \mathcal{W}, \mathcal{E}) \leq 2(K+1)t^{-3}. \quad \forall{t > t_0}, $$

where $t_0 = \inf\left\{\tau \geq 2: \Delta_{\text{min}} \geq 4\sqrt{\frac{2K\log \tau}{\tau}}\right\}$ and $\mathcal{W}$ denotes the event that $m_1(t), m_2(t) \neq k^*$.
\label{lem:suboptimalNotCompetitive}
\end{lem}
\begin{proof}
We now bound this probability as,
\begin{align}
     &\Pr\left((m_1(t) = k \cup m_2(t) = k), n_{k^*}(t) \geq \frac{t}{2K}, \mathcal{E}, \mathcal{W} \right)  \nonumber \\
     &\leq \Pr\left(m_1(t) = k, n_{k^*}(t) \geq \frac{t}{2K}, \mathcal{W}, \mathcal{E} \right) + \Pr\left(m_2(t) = k, n_{k^*}(t) \geq \frac{t}{2K}, \mathcal{W}, \mathcal{E} \right) \\
     &\leq \Pr\left(k = \argmax_{k \in \mathcal{A}_t} I_\ell(t), n_{k^*} \geq \frac{t}{2K}, \mathcal{W}, \mathcal{E}\right) + \Pr\left(m_2(t) = k, n_{k^*} \geq \frac{t}{2K}, \mathcal{W}, \mathcal{E} \right)  \label{eq:hastohappen} \\
    &\leq \Pr\left(\hat{\phi}_{k,k^*} + \sqrt{\frac{2 \log t}{n_{k^*}(t)}} \geq I_{k^*}(t), n_{k^*} \geq \frac{t}{2K}\right) + \Pr\left(m_2(t) = k, n_{k^*} \geq \frac{t}{2K}, \mathcal{W}, \mathcal{E} \right) \\
    &\leq \Pr\left(\hat{\phi}_{k,k^*} + \sqrt{\frac{2 \log t}{n_{k^*}(t)}} \geq I_{k^*}(t), \mu_{k^*} < I_{k^*} n_{k^*} \geq \frac{t}{2K}\right) + \Pr\left(\mu_{k^*} > I_{k^*}(t)\right) + \nonumber \\
    &\quad ~ \quad \Pr\left(m_2(t) = k, n_{k^*} \geq \frac{t}{2K}, \mathcal{W}, \mathcal{E} \right) 
\end{align}
\begin{align}
    &\leq \Pr\left(\hat{\phi}_{k,k^*} + \sqrt{\frac{2 \log t}{n_{k^*}(t)}} \geq \mu_{k^*}, n_{k^*} \geq \frac{t}{2K}\right) + Kt^{-3} + \Pr\left(m_2(t) = k, n_{k^*} \geq \frac{t}{2K}, \mathcal{W}, \mathcal{E} \right) \label{eq:firstHoef}\\
    &= \Pr\left(\hat{\phi}_{k,k^*} - \phi_{k,k^*} \geq \mu_{k^*} - \sqrt{\frac{2 \log t}{n_{k^*}(t)}} , n_{k^*} \geq \frac{t}{2K}\right) + Kt^{-3} + \nonumber \\
    &\quad ~ \quad \Pr\left(m_2(t) = k, n_{k^*} \geq \frac{t}{2K}, \mathcal{W}, \mathcal{E} \right) \\
    &\leq t\exp\left(-2 \frac{t}{2K} \left(\mu_{k^*} - \phi_{k,k^*} - \sqrt{\frac{4K\log t}{t}} \right)^2\right) + Kt^{-3} +  \Pr\left(m_2(t) = k, \mathcal{E} \right) \label{eq:hoeffplusun} \\
    &\leq t^{-3} \exp\left(\Delta_{\text{min}}^2 - 2\Delta_{\text{min}}\sqrt{\frac{4K\log t}{t}} \right) + Kt^{-3} +  \Pr\left(m_2(t) = k, n_{k^*} \geq \frac{t}{2K}, \mathcal{W}, \mathcal{E} \right) \label{eq:invokeNonComp} \\
    &\leq t^{-3} + Kt^{-3} + \Pr\left(m_2(t) = k, \mathcal{W}, \mathcal{E} \right) \quad \forall{t > t_0} \label{eq:firstSubResult}
\end{align}
Here \eqref{eq:firstHoef} follows from \Cref{lem:ucbindexmore1}. Inequality \eqref{eq:hoeffplusun} follows as a result of hoeffding bound and the union bound, as $n_{k^*}$ can take any value between $\frac{t}{2K}$ and $t$ (\Cref{lem:UnionBoundTrickInt}). We get \eqref{eq:invokeNonComp} as $\phi_{k,k^*} < \mu_{k^{(2)}}$ as the arm $k$ is non-competitive.

We now bound $\Pr(\left(m_2(t) = k, \mathcal{E}\right)$ separately. Under $\mathcal{E}$, the crossUCB index $\tilde{U}_{k^*, k^*}\left(n_{k^*}, \frac{\delta}{2K}\right)$ is larger than $\mu_{k^*}$. Using similar steps as done for the first term we now evaluate the upper bound on the probability that arm $k$ to be selected as $m_2(t)$ at round $t$, \begin{align}
    &\Pr\left(m_2(t) = k, n_{k^*} \geq \frac{t}{2K}, \mathcal{W}, \mathcal{E} \right) \leq \\
    &\leq \Pr\left(\hat{\phi}_{k,k^*} + \sqrt{\frac{2 \log t}{n_{k^*}(t)}} \geq \mu_{k^*}, I_{k^*}(t) > \mu_{k^*}, n_{k^*} \geq \frac{t}{2K} \right) + \Pr\left(\mu_{k^*} > I_{k^*}(t)\right) \\
    &\leq t^{-3} + Kt^{-3} \\
\end{align}

Combining this with \eqref{eq:firstSubResult}, we get the result of \Cref{lem:suboptimalNotCompetitive}.
\end{proof}

\begin{lem}
If $\gap_{\text{min}} \geq 4\sqrt{\frac{2\numArms \log \slot_0}{\slot_0}}$ for some constant $\slot_0 > 0$, then,
$$\Pr(m_1(t) = \arm , \pulls_\arm(\slot) \geq s, \mathcal{E}) \leq 2(K+1) \slot^{-3} \quad \text{for } s > \frac{\slot}{2 \numArms}, \forall \slot > \slot_0.$$
\label{lem:noMorePulls}
\end{lem}

\begin{proof}
By noting that $m_1(t) = \arm$ corresponds to arm $\arm$ having the highest pseudoUCB index among the set of active arms at round $t$ (denoted by $\mathcal{A}_t$), we have,  
\begin{align}
    \Pr(m_1(t) = \arm , \pulls_\arm(\slot) \geq s, \mathcal{E}) &= \Pr(\Index_\arm(\slot) = \arg \max_{\arm' \in \mathcal{A}_t} \Index_{\arm'}(\slot) , \pulls_\arm(\slot) \geq s, \mathcal{E})  \\
    &\leq  \Pr\left(\Index_\arm(\slot) > \Index_{\arm^*}(\slot) , \pulls_\arm(\slot) \geq s\right) . \label{bestArmThere}
\end{align}
Here \eqref{bestArmThere} follows from the fact that under $\mathcal{E}$, $k^*$ is always in $\mathcal{A}_t$ (\Cref{subsec:proofthm1}).

\begin{align}
    &\Pr(\Index_\arm(\slot) > \Index_{\arm^*}(\slot) , \pulls_\arm(\slot) \geq s) = \nonumber \\
    &\Pr\left(\Index_\arm(\slot) > \Index_{\arm^*}(\slot) ,  \pulls_\arm(\slot) \geq s, \meanReward_{\arm^*} \leq \Index_{\arm^*}(\slot)\right)  + \nonumber \\
    &\quad \Pr\left(\Index_\arm(\slot) > \Index_{\arm^*}(\slot), \pulls_\arm(\slot) \geq s | \meanReward_{\arm^*} > \Index_{\arm^*}(\slot) \right) \times \Pr\left(\meanReward_{\arm^*} > \Index_{\arm^*}(\slot) \right) \label{conditionTerm} \\
    &\leq  \Pr\left(\Index_\arm(\slot) > \Index_{\arm^*}(\slot), \pulls_\arm(\slot) \geq s, \meanReward_{\arm^*} \leq \Index_{\arm^*}(\slot)\right) + \Pr\left(\meanReward_{\arm^*} > \Index_{\arm^*}(\slot)\right) \label{droppingTerms}\\
    &\leq \Pr\left(\Index_{k,k}(\slot) > \Index_{\arm^*}(\slot), \pulls_\arm(\slot) \geq s, \meanReward_{\arm^*} \leq \Index_{\arm^*}(\slot)\right) + K\slot^{-3} \label{usingHoeffdingAgain}\\
    &= \Pr\left(\Index_{k,k}(\slot) > \meanReward_{\arm^*} ,  \pulls_\arm(\slot) \geq s\right) + K\slot^{-3} \label{usingConditioning} \\
    &= \Pr\left(\hat{\meanReward}_\arm(\slot) + \sqrt{\frac{2 \log \slot}{\pulls_\arm(\slot)}} > \meanReward_{\arm^*} , \pulls_\arm(\slot) \geq s \right) + K\slot^{-3} \label{expandingIndex}\\
    &= \Pr\left(\hat{\meanReward}_\arm(\slot) - \meanReward_\arm > \meanReward_{\arm^*} - \meanReward_\arm - \sqrt{\frac{2 \log \slot}{\pulls_\arm(\slot)}} , \pulls_\arm(\slot) \geq s \right) + K\slot^{-3} \\ 
    &= \Pr\left( \frac{\sum_{\tau = 1}^{\slot} \indicator_{\{\arm_\tau = \arm\}}\reward_\tau}{\pulls_\arm(\slot)} - \meanReward_\arm > \gap_\arm - \sqrt{\frac{2 \log \slot}{\pulls_\arm(\slot)}} , \pulls_\arm(\slot) \geq s\right) + K\slot^{-3} \\
    &\leq \slot \exp\left(-2 s \left(\gap_\arm - \sqrt{\frac{2 \log \slot}{s}}\right)^2\right) + K\slot^{-3} \label{eqn:chernoffagain}\\
    &\leq \slot^{-3}\exp\left(-2 s \left(\gap_\arm^2 - 2 \gap_\arm \sqrt{\frac{2 \log \slot}{s}}\right)\right) + K\slot^{-3} \\
    &\leq 2 (K+1) \slot^{-3} \quad \text{ for all  } \slot > \slot_0. \label{finalCondn}
\end{align}
We have \eqref{conditionTerm} holds because of the fact that $P(A) = P(A|B)P(B) + P(A|B^c)P(B^c)$, Inequality \eqref{usingHoeffdingAgain} follows from \Cref{lem:ucbindexmore1} and from the fact that $I_{k}(t) = \min_{\ell}I_{k,\ell}(t)$. From the definition of $\Index_{k,k}(\slot)$ we have \eqref{expandingIndex}. Inequality \eqref{eqn:chernoffagain} follows from Hoeffding's inequality and the term $\slot$ before the exponent in \eqref{eqn:chernoffagain} arises as the random variable $\pulls_\arm(\slot)$ can take values from $s$ to $\slot$ (\Cref{lem:UnionBoundTrickInt}). Inequality \eqref{finalCondn} follows from the fact that $s > \frac{\slot}{2 \numArms}$ and $\gap_\arm \geq 4\sqrt{\frac{2\numArms \log \slot_0}{\slot_0}}$ for some constant $\slot_0 > 0.$

\end{proof}

\begin{lem}
Let $n^{m_1}_k(t)$ denote the number of times arm $k$ has been sampled as $m_1(t)$ till round $t$. If $\gap_{\text{min}} \geq 4\sqrt{\frac{2 \numArms \log \slot_0}{\slot_0}}$ for some constant $\slot_0 > 0$, then, $$\Pr\left(\pulls_\arm^{m_1}(\slot) > \frac{t}{ \numArms}, \mathcal{E} \right) \leq  \frac{(K+1)^3}{t^2} \quad \forall \slot > \numArms \slot_0.$$
\label{lem:suboptimalNotPulled}
\end{lem}

\begin{proof}
We expand $\Pr\left(\pulls_\arm(\slot) > \frac{t}{\numArms}\right)$ as,

\begin{align}
    \Pr\left(\pulls^{m_1}_\arm(\slot) \geq \frac{\slot}{\numArms}, \mathcal{E} \right) &= \Pr\left( \pulls^{m_1}_{\arm}(\slot) \geq \frac{\slot}{\numArms}, \mathcal{E} \Big| \pulls^{m_1}_\arm(\slot - 1) \geq \frac{\slot}{\numArms}, \mathcal{E} \right) \Pr\left( \pulls^{m_1}_\arm(\slot - 1) \geq \frac{\slot}{\numArms}, \mathcal{E} \right) + \nonumber \\
    &\quad \Pr\left(m_1(t) = \arm , \pulls^{m_1}_\arm(\slot - 1) = \frac{\slot}{\numArms} - 1, \mathcal{E} \right)  \\
    &\leq \Pr\left(\pulls^{m_1}_\arm(\slot - 1) \geq \frac{\slot}{\numArms}, \mathcal{E} \right) + \Pr\left(m_1(t) = \arm , \pulls^{m_1}_\arm(\slot - 1) = \frac{\slot}{\numArms} - 1, \mathcal{E} \right) \\
    &\leq \Pr\left(\pulls^{m_1}_\arm(\slot - 1) \geq \frac{\slot}{\numArms}, \mathcal{E} \right) + (2K + 2) (\slot - 1)^{-3} \quad \forall (\slot - 1) > \slot_0. \label{fromPrevLemma}
\end{align}
Here, \eqref{fromPrevLemma} follows from \Cref{lem:noMorePulls}.\\

This gives us $$\Pr\left(\pulls^{m_1}_\arm(\slot) \geq \frac{\slot}{\numArms}, \mathcal{E} \right) - \Pr\left(\pulls^{m_1}_\arm(\slot - 1) \geq \frac{\slot}{\numArms}, \mathcal{E} \right) \leq (2K + 2)(\slot - 1)^{-3}, \quad \forall (\slot - 1) > \slot_0.$$
Now consider the summation $$ \sum_{\tau = \frac{\slot}{\numArms}}^{\slot} \Pr\left(\pulls^{m_1}_\arm(\tau) \geq \frac{\slot}{\numArms}, \mathcal{E} \right) - \Pr\left(\pulls^{m_1}_\arm(\tau - 1) \geq \frac{\slot}{\numArms}, \mathcal{E} \right) \leq \sum_{\tau = \frac{\slot}{\numArms}}^{\slot}(2K + 2)(\tau - 1)^{-3}.$$ This gives us, $$\Pr\left(\pulls^{m_1}_\arm(\slot) \geq \frac{\slot}{\numArms}, \mathcal{E} \right) - \Pr\left(\pulls^{m_1}_\arm\left(\frac{\slot}{\numArms} - 1 \right) \geq \frac{\slot}{\numArms}, \mathcal{E} \right) \leq \sum_{\tau = \frac{\slot}{\numArms}}^{\slot}(2K + 2)(\tau - 1)^{-3}.$$
Since $\Pr\left(\pulls^{m_1}_\arm\left(\frac{\slot}{\numArms} - 1 \right)\geq \frac{\slot}{\numArms}, \mathcal{E}\right)  = 0$, we have, 
\begin{align}
    \Pr\left(\pulls^{m_1}_\arm(\slot) \geq \frac{\slot}{\numArms}, \mathcal{E} \right) &\leq \sum_{\tau = \frac{\slot}{\numArms}}^{\slot}(2K + 2)(\tau - 1)^{-3} \\
    &\leq (K + 1) \left(\frac{\slot}{\numArms} - 2\right)^{-2} \quad \forall \slot > \numArms \slot_0. \label{eq:tookintegral}
\end{align}
The last step \eqref{eq:tookintegral} follows from the fact that $\sum_{\tau = t/K}^{t} (\tau - 1)^{-3} \leq \int_{\tau = t/K - 1}^{\infty} (\tau - 1)^{-3}$.
\end{proof}

\section{Probability of sampling a non-competitive arm at round $t$}
For ease of presentation we denote $\mathcal{W}$ to be the event that $m_1(t), m_2(t) \neq k^*$.
\begin{lem}
The probability of sampling a non-competitive arm at round $t$, jointly with the event $\mathcal{E}$, is bounded as 
$$\Pr\left((m_1(t) \notin \mathcal{C} \cup m_2(t) \notin \mathcal{C}), \mathcal{W}, \mathcal{E}\right) \leq \frac{2(K + 1)K}{t^3} + \frac{K(K+1)^3}{t^2} \quad \forall{t > Kt_0}.$$
\label{lem:probNonComp}
\end{lem}

\begin{proof}
\begin{align}
&\Pr((m_1(t) \notin \mathcal{C} \cup m_2(t) \notin \mathcal{C}), \mathcal{E}) =  \nonumber \\
&\Pr\left((m_1(t) \notin \mathcal{C} \cup m_2(t) \notin \mathcal{C}), \mathcal{W}, \mathcal{E}, n_{k^*}(t) \geq \frac{t}{K}\right) + \nonumber \\
&\quad ~ \quad ~ \quad \quad \Pr\left((m_1(t) \notin \mathcal{C} \cup m_2(t) \notin \mathcal{C}), \mathcal{W}, \mathcal{E}, n_{k^*}(t) < \frac{t}{K}\right) \\
&\leq \Pr\left((m_1(t) \notin \mathcal{C} \cup m_2(t) \notin \mathcal{C}), \mathcal{W}, \mathcal{E}, n_{k^*}(t) \geq \frac{t}{K}\right) + \Pr\left(n_{k^*}(t) < \frac{t}{K}, \mathcal{E}\right) \\
&\leq \Pr\left((m_1(t) \notin \mathcal{C} \cup m_2(t) \notin \mathcal{C}),\mathcal{W},  \mathcal{E}, n_{k^*}(t) \geq \frac{t}{K}\right) + \Pr\left(n^{m_1}_{k^*}(t) < \frac{t}{K}, \mathcal{E}\right) \label{whatism1} \\
&\leq \Pr\left((m_1(t) \notin \mathcal{C} \cup m_2(t) \notin \mathcal{C}),\mathcal{W}, \mathcal{E}, n_{k^*}(t) \geq \frac{t}{K}\right) + \sum_{k \neq k^*} \Pr\left(n^{m_1}_k(t) \geq \frac{t}{K}, \mathcal{E} \right) \\
&\leq \sum_{k \notin \mathcal{C}} \Pr\left((m_1(t) = k \cup m_2(t) = k), \mathcal{W}, \mathcal{E}, n_{k^*}(t) \geq \frac{t}{K}\right) + \sum_{k \neq k^*} \Pr\left(n^{m_1}_k(t) \geq \frac{t}{K}, \mathcal{E} \right) \\
&\leq \frac{2(K+1)K}{t^3} + \frac{K(K+1)^3}{t^2} \quad \forall{t > Kt_0} 
\end{align}
In \eqref{whatism1}, $n^{m_1}_{k^*}(t)$ denotes the number of times arm $k^*$ was samples as $m_1(t)$ till round $t$. As $n^{m_1}_{k^*}(t) < n_{k^*}(t)$, we have \eqref{whatism1}. The last step follows from \Cref{lem:suboptimalNotCompetitive} and \Cref{lem:suboptimalNotPulled}.
\end{proof}

\section{Intermediate steps to analyse samples obtained from competitive arms}

For $k \neq k^*$, define $\tau_k$ to be the first integer such that $\Confidence\left(n_k, \frac{\delta}{2K}\right) < \frac{\Delta_k}{4}$ and define $\tau_{k^*} = \tau_{k^{(2)}}$. We call an arm $k$ to be GOOD at round $t$, if $\Confidence\left(n_k, \frac{\delta}{2K}\right) \leq \frac{\Delta_k}{4}$, i.e., an arm is GOOD if it has been sampled \emph{significant} number of times till round $t$, i.e., $n_k(t) \geq \tau_k$. Otherwise, the arm is called BAD. We denote $\mu^{\text{ref}}$ as $\frac{\mu_{k^*} + \mu_{k^{(2)}}}{2}$, i.e., the average of the mean reward of best and second best arm. We will first show that an arm $k \neq k^*$ being GOOD implies that its psuedoUCB index is below $\mu^{\text{ref}}$, i.e., $n_{k} > \tau_k \Rightarrow \tilde{U}_{k,k}\left(n_k, \frac{\delta}{2K}\right) < \mu^{\text{ref}}$. Consider $\tilde{U}_{k,k}\left(n_k, \frac{\delta}{2K}\right)$ for $k \neq k^*, n_k > \tau_k$. Under $\mathcal{E}$, we have 
\begin{align}
    \hat{\mu}_k + \Confidence\left(n_k, \frac{\delta}{2K}\right) &\leq \mu_k + 2\Confidence\left(n_k, \frac{\delta}{2K}\right) \label{eq:firSte} \\
    &= \mu^{\text{ref}} + 2\Confidence\left(n_k, \frac{\delta}{2K}\right) + \frac{(\mu_k - \mu_{k^*}) + (\mu_k - \mu_{k^{(2)}})}{2} \\
    &\leq \mu^{\text{ref}} + 2\Confidence\left(n_k, \frac{\delta}{2K}\right) - \frac{\Delta_k}{2} \\
    &\leq \mu^{\text{ref}} \label{eq:lastSte}
\end{align}
Here \eqref{eq:firSte} follows from the fact that, under $\mathcal{E}$, $\hat{\mu}_k \leq \mu_{k} + B_k\left(n_k, \frac{\delta}{2K}\right)$. The last step follows as arm $k$ is GOOD, i.e., $\Confidence\left(n_k, \frac{\delta}{2K}\right) \leq \frac{\Delta_k}{4}$.

Using a similar argument for $k^*$, we can prove that Arm $k^*$ being GOOD, i.e., $n_{k^*} > \tau_{k^*} \Rightarrow L_{k^*}\left(n_{k^*}, \frac{\delta}{2K}\right) > \mu^{\text{ref}},$. In addition to this, $n_{k^*} > \tau_{k^*}$ (i.e., Arm $k^*$ being GOOD), also implies that $\tilde{U}_{k, k^*} < \mu^{\text{ref}}$ for $k \notin \mathcal{C}$ as we present below. Under $\mathcal{E}$, we have the bound on $\tilde{U}_{k,k^*}\left(n_{k^*},\frac{\delta}{2K}\right) = \hat{\phi}_{k,k^*} + \Confidence\left(n_{k^*}, \frac{\delta}{2K}\right)$, as follows,
\begin{align}
    \hat{\phi}_{k,k^*} + \Confidence\left(n_{k^*}, \frac{\delta}{2K}\right) &\leq \phi_{k,k^*} + 2\Confidence\left(n_{k^*}, \frac{\delta}{2K}\right) \\
    &\leq \mu_{k^{(2)}} + 2\Confidence\left(n_{k^*}, \frac{\delta}{2K}\right) \label{eqn:secIneq} \\
    &\leq \mu_{k^{(2)}} + 2 \frac{\Delta_{\text{min}}}{4} \\
    &\leq \mu^{\text{ref}} \label{eq:invokeThis}
\end{align}
The inequality \eqref{eqn:secIneq} follows from the fact that arm $k \notin \mathcal{C}$, i.e., $\phi_{k,k^*} < \mu_{k^{(2)}}$. We now use this observation to list four possible scenarios under which algorithm does not stop and bound each individual term to prove the statement of \Cref{thm:sampComp}.

Define $\mathcal{R}(t)$ to be the event that $I_{k^*}(t) < \mu_{k^*}$, i.e., $\mathcal{R}(t) = \{I_{k^*}(t) > \mu_{k^*}\}$. By \Cref{lem:ucbindexmore1}, $\Pr(\mathcal{R}(t)) \leq Kt^{-3}$. 

\begin{lem}
If the algorithm has not stopped at round $t$ and the event $\mathcal{E}$ holds true, at least one of the following occurs
\begin{enumerate}
    \item Event $\mathcal{R}(t)$ does not occur,
    \item $m_1(t)$ or $m_2(t)$ is Non-Competitive and $m_1(t), m_2(t) \neq k^*$ 
    \item ($m_1(t) = k^*$ is BAD and $m_2(t) \notin \mathcal{C}$) or ($m_2(t) = k^*$ is BAD and $m_1(t) \notin \mathcal{C}$)
    \item $m_1(t), m_2(t) \in \mathcal{C}$ and either $m_1(t)$ is BAD or $m_2(t)$ is BAD.
\end{enumerate}
\label{lem:conditions}
\end{lem}
\begin{proof}
We prove this by contradiction. We consider the event that all the four cases listed above do not occur jointly and show that such a situation cannot occur if algorithm has not stopped till round $t$ under $\mathcal{E}$. The proof technique is inspired from the analysis done in \cite{jamieson2014best} but needed some modification to prove the result for C-LUCB algorithm in a correlated bandit environment. Let's break down the scenario where all of the four events listed in \Cref{lem:conditions} do not occur and look at each of them individually.

\noindent
\textbf{Case 1:} \\ $\{m_1(t) = k^*,$ $m_1(t)$ is GOOD $\} \cap \{m_2(t) \neq k^*, m_2(t) \in \mathcal{C}$, $m_2(t)$ is GOOD $\} \cap \mathcal{R}(t) \cap \{t < \mathcal{T}\} $.

We note the following two things in this case, 
\begin{enumerate}
    \item $m_1(t) = k^*$ is GOOD $\Rightarrow$ $L_{k^*}\left(n_{k^*}, \frac{\delta}{2K}\right) > \mu^{\text{ref}}$. 
    \item $m_2(t) = \ell \neq k^*$ is GOOD $\Rightarrow$ $\tilde{U}_{\ell, \ell}\left(n_{\ell}, \frac{\delta}{2K}\right) < \mu^{\text{ref}}$.
\end{enumerate}

\noindent
As we have, $L_{k^*}\left(n_{k^*}, \frac{\delta}{2K}\right) > \tilde{U}_{\ell}\left(n_{\ell}, \frac{\delta}{2K}\right)$ at round $t$, arm $\ell$ cannot belong the the set of active arms $\mathcal{A}_t$ and hence cannot be selected as $m_2(t)$.

\noindent 
\textbf{Case 2:} \\ $\{m_1(t) \neq k^*, m_1(t) \in \mathcal{C},$ $m_1(t)$ is GOOD $\} \cap \{m_2(t) = k^*$, $m_2(t)$ is GOOD $\} \cap \mathcal{R}(t) \cap \{t < \mathcal{T}\}$.

In case 2, we make the following observations 
\begin{enumerate}
    \item $m_1(t) = \ell \neq k^*$ is GOOD $\Rightarrow$ $\tilde{U}_{\ell,\ell}\left(n_{\ell}, \frac{\delta}{2K}\right) < \mu^{\text{ref}}$. 
    \item $m_2(t) = k^*$ is GOOD $\Rightarrow$ $L_{k^*}\left(n_{k^*}, \frac{\delta}{2K}\right) > \mu^{\text{ref}}$. 
\end{enumerate}

\noindent
As we have, $L_{k^*}\left(n_{k^*}, \frac{\delta}{2K}\right) > \tilde{U}_{\ell}\left(n_{\ell}, \frac{\delta}{2K}\right)$ at round $t$, arm $\ell$ cannot belong the the set of active arms $\mathcal{A}_t$ and hence cannot be selected as $m_1(t)$.

\noindent
\textbf{Case 3:} \\ $\{m_1(t) \neq k^*, m_1(t) \in \mathcal{C},$ $m_1(t)$ is GOOD$\} \cap \{m_2(t) \neq k^*, m_2(t) \in \mathcal{C},$ $m_2(t)$ is GOOD$\} \cap \mathcal{R}(t) \cap \{t < \mathcal{T}\}$.

For case 3, we see that 
\begin{enumerate}
    \item $m_1(t) = \ell_1 \neq k^*$ is GOOD $\Rightarrow$  $\tilde{U}_{\ell_1, \ell_1}\left(n_{\ell_1}, \frac{\delta}{2K}\right) < \mu^{\text{ref}}$.
    \item $m_2(t) = \ell_2 \neq k^*$ is GOOD $\Rightarrow$  $\tilde{U}_{\ell_2, \ell_2}\left(n_{\ell_2}, \frac{\delta}{2K}\right) < \mu^{\text{ref}}$, it further implies that \\ $\min\left(I_{\ell_2}(t), \tilde{U}_{\ell_2, \ell_2}\left(n_{\ell_2}, \frac{\delta}{2K}\right)\right) \leq \mu^{\text{ref}}$.
\end{enumerate}
As arm $k^*$ is not selected, it implies that either $I_{k^*}(t) \leq \mu^{\text{ref}}$ or $\tilde{U}_{k^*, k^*}\left(n_{k^*}, \frac{\delta}{2K}\right) \leq \mu^{\text{ref}}$. By $\mathcal{R}(t)$, $I_{k^*}(t) \geq \mu_{k^*} > \mu^{\text{ref}}$ and with event $\mathcal{E}$, $\tilde{U}_{k^*, k^*}\left(n_{k^*}, \frac{\delta}{2K}\right) > \mu_{k^*} > \mu^{\text{ref}}$. This shows that case 3 cannot occur and leads to a contradiction.

\noindent
\textbf{Case 4:} \\ $\{(m_1(t) = k^*$ is GOOD, $m_2(t) = \ell \notin \mathcal{C}) \cup (m_2(t) = k^*$ is GOOD, $m_1(t) = \ell \notin \mathcal{C})\} \cap \mathcal{R}(t) \cap \{t < \mathcal{T}\}.$

For Case 4, we see from \eqref{eq:lastSte}, \eqref{eq:invokeThis} that 
\begin{enumerate}
    \item $k^*$ is GOOD $\Rightarrow$ $L_{k^*}\left(n_{k^*}, \frac{\delta}{2K}\right) > \mu^{\text{ref}}$.
    \item $k^*$ is GOOD $\Rightarrow$ $\tilde{U}_{\ell, k^*}\left(n_{k^*}, \frac{\delta}{2K}\right) < \mu^{\text{ref}}$.
\end{enumerate}
As $\tilde{U}_{\ell}\left(\frac{\delta}{2K}\right) < \tilde{U}_{\ell,k^*}\left(n_{k^*}, \frac{\delta}{2K}\right) < \mu^{\text{ref}} < L_{k^*}\left(n_{k^*}, \frac{\delta}{2K}\right)$, arm $\ell$ cannot be in the set of active arms at round $t$ and hence cannot be sampled at round $t$. Therefore, all the four cases listed above cannot occur and we have a contradiction.

\noindent
This proves the statement of \Cref{lem:conditions}, as at least one of the events listed in \Cref{lem:conditions} must occur for the algorithm to proceed further. This analysis follows similar steps as that in \cite{jamieson2014best, kalyanakrishnan2012pac} but needed further modifications to prove statement for our C-LUCB algorithm.
\end{proof}

\begin{lem}
Let $T^{(B)}$ denote the total number of times that the events $(3)$ or $(4)$ of \Cref{lem:conditions} occur. We have that $T^{(B)}$ is upper bounded by $\sum_{k \in \mathcal{C}} \frac{\constant}{\Delta_k^2} \log \left(\frac{2K \log\left(\frac{1}{\Delta_k^2}\right)}{\delta}\right)$ under the event $\mathcal{E}$.
\label{lem:countingBad}
\end{lem}
\begin{proof}

We now bound $T^{(B)}$ under the event $\mathcal{E}$, 
\begin{align}
    T^{(B)} &= \sum_{t = 1}^{\infty} \indicator \Big(\{m_1(t) = k^* \text{ is BAD}, m_2(t) \notin \mathcal{C} \} \cup \{m_2(t) = k^* \text{is BAD}, m_1(t) \notin \mathcal{C} \} \bigcup \nonumber \\
    &\{m_1(t) \in \mathcal{C} \text{ is BAD  or } m_2(t) \in \mathcal{C} \text{ is BAD} \} \Big) \\
    &\leq \sum_{t = 1}^{\infty} \indicator \Big(\left(\{m_1(t) \text{ is }k^* \text{ or } m_2(t) \text{ is }k^* \} \cap \{k^* \text{ is BAD} \}\right) \bigcup \nonumber \\
    &\left(\{m_1(t) \in \mathcal{C} \text{ is BAD  or } m_2(t) \in \mathcal{C} \text{ is BAD} \} \right)\Big) \\
    &= \sum_{t = 1}^{\infty} \sum_{k \in \mathcal{C}} \indicator \Big(\left(\{m_1(t) \text{ is }k^* \text{ or } m_2(t) \text{ is }k^* \} \cap \{k^* \text{ is BAD} \}\right) \bigcup \nonumber \\
    &= \left( \{ m_1(t) \text{ is k or } m_2(t) \text{ is k} \} \cap \{k \text{ is BAD } \}\right) \Big)\\
    &= \sum_{t = 1}^{\infty} \sum_{k \in \mathcal{C}} \indicator \left( \{ m_1(t) \text{ is k or } m_2(t) \text{ is k} \} \cap \{k \text{ is BAD } \}\right) \\
    &= \sum_{t = 1}^{\infty} \sum_{k \in \mathcal{C}} \indicator \left( \{ m_1(t) \text{ is k or } m_2(t) \text{ is k} \} \cap \{n_k(t) \leq \tau_k \}\right) \\
    &\leq \sum_{k \in \mathcal{C}} \tau_k  \label{eq:lastineq}
\end{align}
The last \eqref{eq:lastineq} holds from the fact that if $n_k(t) \leq \tau_k$ and $m_1(t)$ is $k$ or $m_2(t)$ is $k$, then arm $k$ gets sampled and $n_k(t+1) = n_k(t) + 1$, this can only occur $\tau_k$ times before $n_k(t) > \tau_k$. For anytime confidence intervals $\Confidence \left(n_k, \frac{\delta}{2K}\right)$, first integer $\tau_k$ such that $\Confidence \left(n_k, \frac{\delta}{2K}\right) < \frac{\Delta_k}{4}$ is upper bounded by $\frac{\constant}{\Delta_k^2} \log \left(\frac{2K \log\left(\frac{1}{\Delta_k^2}\right)}{\delta}\right)$ where $\constant > 0$ is a constant depending on the tightness of confidence interval $\Confidence (n_k, \delta)$ \cite{simchowitz2017simulator}. The tighter the confidence interval, smaller is the constant $\constant$. Due to this, we get a bound on $T^{(B)}$ under the event $\mathcal{E}$ as, $$T^{(B)} \leq \sum_{k \in \mathcal{C}} \frac{\constant}{\Delta_k^2} \log \left(\frac{2K \log\left(\frac{1}{\Delta_k^2}\right)}{\delta}\right).$$

As the probability of event $\mathcal{E}$ is at least $1 - \delta$, we get that $T^{(B)} \leq \sum_{k \in \mathcal{C}} \frac{\constant}{\Delta_k^2} \log \left(\frac{2K \log\left(\frac{1}{\Delta_k^2}\right)}{\delta}\right)$ with probability $1 - \delta$. In \Cref{sec:regret}, we denoted $T^{(C)}$ as the total number of rounds in which $m_1(t), m_2(t) \in \mathcal{C}$ and $I_{k^*}(t) > \mu_{k^*}$ and similarly $T^{(*)}$ as the total number of rounds in which $m_1(t) = k^*, m_2(t) \notin \mathcal{C}$ or $m_2(t) = k^*, m_1(t) \notin \mathcal{C}$. From \Cref{lem:conditions}, we note that $T^{(*)} + T^{(C)}$ is equivalent to $T^{(B)}$ on which we derived a bound above.  Due to this, $T^{(C)} + T^{(*)} = T^{(B)} \leq \sum_{k \in \mathcal{C}} \frac{\constant}{\Delta_k^2} \log \left(\frac{2K \log\left(\frac{1}{\Delta_k^2}\right)}{\delta}\right)$ under the event $\mathcal{E}$. 

\end{proof}

\section{Proof of Theorem 2}
\label{subsec:proofthm2}
We now bound the total number of rounds played by C-LUCB algorithm under the event $\mathcal{E}$. From \Cref{lem:conditions}, we note that if the algorithm has not stopped at round $t$ under the event $\mathcal{E}$, it implies that at least one of the following events must be true at round $t$,
\begin{enumerate}
    \item Event $\mathcal{R}(t)$ does not occur, i.e., $I_{k^*}(t) < \mu_k^*$
    \item $m_1(t)$ or $m_2(t)$ is Non-Competitive and $m_1(t), m_2(t) \neq k^*$, 
    \item ($m_1(t) = k^*$ is BAD and $m_2(t) \notin \mathcal{C}$) or ($m_2(t) = k^*$ is BAD and $m_1(t) \notin \mathcal{C}$)
    \item $m_1(t), m_2(t) \in \mathcal{C}$ and either $m_1(t)$ is BAD or $m_2(t)$ is BAD.
\end{enumerate}
From \Cref{lem:ucbindexmore1} we see that $\Pr(\mathcal{R}(t)) \leq \frac{K}{t^3}$ and the result from \Cref{lem:probNonComp} gives us a bound on $\Pr((m_1(t) \notin \mathcal{C} \cup m_2(t) \notin \mathcal{C}), (m_1(t), m_2(t) \neq k^*), \mathcal{E})$. The result from \Cref{lem:countingBad} shows that the third and fourth event occur at most $\sum_{k \in \mathcal{C}} \frac{\constant}{\Delta_k^2} \log \left(\frac{2K \log\left(\frac{1}{\Delta_k^2}\right)}{\delta}\right)$ times. Combining these, we get our desired bound on the sample complexity result. 

\begin{proof}
\begin{align}
    \E{\mathcal{T} | \mathcal{E}} &= \sum_{t = 1}^{\infty} \E{\indicator(t = \mathcal{T} | E)} \\
    &\leq \sum_{t = 1}^{\infty} \indicator(\mathcal{R}(t) | \mathcal{E})  + \E{T^{(B)} | \mathcal{E})} + \nonumber\\
    &\quad \sum_{t = 1}^{\infty} \E{\indicator(m_1(t) \text{ or } m_2(t) \notin \mathcal{C},  m_1(t) \text{ and } m_2(t) \neq k^* | \mathcal{E})} \\
    &= \sum_{t = 1}^{\infty} \Pr(\mathcal{R}(t), \mathcal{E}) \times \frac{1}{\Pr(\mathcal{E})} + \E{T^{(B)} | \mathcal{E})} + \nonumber \\
    &\quad \sum_{t = 1}^{\infty} \Pr\left((m_1(t) \text{ or } m_2(t) \notin \mathcal{C}), m_1(t) \text{ and } m_2(t) \neq k^*, \mathcal{E} \right) \times \frac{1}{\Pr(\mathcal{E})} \\
    &\leq \sum_{t = 1}^{\infty} \frac{1}{1 - \delta} \times \frac{K}{t^3} + \sum_{k \in \mathcal{C}} \frac{\constant}{\Delta_k^2} \log \left(\frac{2K \log\left(\frac{1}{\Delta_k^2}\right)}{\delta}\right) + \nonumber \\ 
    &\quad \frac{Kt_0}{1 - \delta} + \frac{1}{1-\delta}\sum_{t = Kt_0 + 1}^{\infty} \left(\frac{2(K+1)K}{t^3} + \frac{K(K+1)^3}{t^2}\right) \\
    &\leq \frac{3K}{2(1 - \delta)} + \frac{Kt_0}{1 - \delta} + \frac{1}{1-\delta} \times \left(\frac{2}{t_0^2} + \frac{(K+1)^3}{t_0}\right) + \sum_{k \in \mathcal{C}} \frac{\constant}{\Delta_k^2} \log \left(\frac{2K \log\left(\frac{1}{\Delta_k^2}\right)}{\delta}\right)
\end{align}
\end{proof}
By noting that the C-LUCB samples two arms at each round, we get the sample complexity result stated in \Cref{thm:sampComp}.
\section{Proof for Theorem 1}
\label{subsec:proofthm1}
\begin{proof}
 To prove theorem 1, we define three events $\mathcal{E}_1, \mathcal{E}_2$ and $\mathcal{E}_3$ below.  Let $\mathcal{E}_1$ be the event that empirical mean of all arm lie within their confidence intervals uniformly for all $t \geq 1$
\begin{equation}
\mathcal{E}_1 = \Bigg\{ \forall{t \geq 1}, \forall{k \in \mathcal{K}}, ~~~ \hat{\mu}_k(t) - \Confidence \left(n_k(t), \frac{\delta}{2K}\right) \leq \mu_k \leq \hat{\mu}_k + \Confidence \left(n_k(t), \frac{\delta}{2K}\right) \Bigg\}    
\label{assump:wellBehaved}
\end{equation} 

\noindent
Define $\mathcal{E}_2$ to be the event that empirical pseudo-reward of optimal arm with respect to all other arms lie within their crossUCB indices uniformly for all $t \geq 1$, i.e., 
\begin{equation}
\mathcal{E}_2 = \Bigg\{ \forall{t \geq 1}, \forall{\ell \in \mathcal{K}}, ~~~ \phi_{k^*,\ell} \leq \hat{\phi}_{k^*,\ell}(t) + \Confidence \left(n_\ell(t), \frac{\delta}{2K}\right) \Bigg\}    
\label{assump:PseudowellBehaved}
\end{equation} 

Similarly define $\mathcal{E}_3$ to be the event that the empirical pseudo-reward of the sub-optimal arms with respect to the optimal arm lies within their crossUCB indices uniformly for all $t \geq 1$, i.e., 
\begin{equation}
\mathcal{E}_3 = \Bigg\{ \forall{t \geq 1}, \forall{\ell \in \mathcal{K}}, ~~~ \phi_{\ell, k^*} \leq \hat{\phi}_{\ell, k^*}(t) + \Confidence \left(n_{k^*}(t), \frac{\delta}{2K}\right) \Bigg\}    
\label{assump:PseudowellBehavedsecondbest}
\end{equation} 
\noindent
Furthermore, we define $\mathcal{E}$ to be the intersection of the three events, i.e., 
\begin{equation}
\mathcal{E} = \mathcal{E}_1 \cap \mathcal{E}_2 \cap \mathcal{E}_3.
\label{eq:defEappendix}
\end{equation}

\noindent
Due to the nature of anytime confidence intervals (See \Cref{eq:anytime}) and union bound over the set of arms, we have $\Pr(\mathcal{E}^c_1) \leq \frac{\delta}{2}$, $\Pr(\mathcal{E}_2^c) \leq \frac{\delta}{4}$ and $\Pr(\mathcal{E}^c_3) \leq \frac{\delta}{4}$ giving us $\Pr(\mathcal{E}^{c}) \leq \delta$. We now show that under the event $\mathcal{E}$, the C-LUCB algorithm cannot stop with an arm $k \neq k^*$. We do that through a proof by contradiction.

Suppose, the algorithm stops with arm $k \neq k^*$, i.e., arm $k$ is the only arm in set $\mathcal{A}_t$. In such a scenario, $\exists \tau, k \neq k^*: \tilde{U}_{k^*}\left(\tau, \frac{\delta}{2K}\right) < L_k\left(\tau, \frac{\delta}{2K}\right)$. This can only occur if one of the following events occur, 
\begin{enumerate}
    \item $\tilde{U}_{k^*,k^*}\left(\tau, \frac{\delta}{2K}\right) <L_k\left(\tau, \frac{\delta}{2K}\right))$
    \item $\tilde{U}_{k^*,\ell}\left(\tau, \frac{\delta}{2K}\right) < L_k\left(\tau, \frac{\delta}{2K}\right) \quad \ell \neq k^*$
\end{enumerate}

See that under the event $\mathcal{E}$, $\tilde{U}_{k^*,\ell}\left(\tau, \frac{\delta}{2K}\right) > \mu_{k^*}$ and $L_k\left(\tau, \frac{\delta}{2K}\right) < \mu_k \quad \forall{\tau,k}$. This shows that under the event $\mathcal{E}$, $\tilde{U}_{k^*}\left(\tau, \frac{\delta}{2K}\right) > L_k\left(\tau, \frac{\delta}{2K}\right) \quad \forall{k, \tau}$ as $\mu_{k^*} > \mu_{k} \quad \forall{k \neq k^*}$. This implies that the algorithm returns the best arm with probability at least $1 - \delta$ as $\Pr(\mathcal{E}^{c}) \leq \delta$.
\end{proof}

\section{$1 - \delta$ Correctness of C-LUCB++}
\label{subsec:plusplus}
We now show that the C-LUCB++ algorithm declares the arm $k^*$ as the best arm with probability at least $1 - \delta$. 
\begin{proof}
 To prove the correctness of C-LUCB++, we use similar arguments as done in the proof of Theorem 1 for the C-LUCB algorithm. In particular, we define an event $\mathcal{E}^+$ that holds true with at least $1-\delta$ probability and show that the C-LUCB++ algorithm always stops with the best arm under the event $\mathcal{E}^+$. 
 
 We define three events $\mathcal{E}^+_1, \mathcal{E}^+_2$ and $\mathcal{E}^+_3$ below.  Let $\mathcal{E}^+_1$ be the event that empirical mean of all arm $k \neq k^*$ lie within their confidence intervals uniformly for all $t \geq 1$
\begin{equation}
\mathcal{E}^+_1 = \Bigg\{ \forall{t \geq 1}, \forall{k \in \mathcal{K}}, ~~~ \hat{\mu}_k(t) - \Confidence \left(n_k(t), \frac{\delta}{3K}\right) \leq \mu_k \leq \hat{\mu}_k + \Confidence \left(n_k(t), \frac{\delta}{3K}\right) \Bigg\}    
\label{assump:wellBehaved1}
\end{equation} 

\noindent
Define $\mathcal{E}^+_2$ to be the event that empirical pseudo-reward of optimal arm with respect to all other arms lie within their confidence intervals uniformly for all $t \geq 1$, i.e., 
\begin{equation}
\mathcal{E}^+_2 = \Bigg\{ \forall{t \geq 1}, \forall{\ell \in \mathcal{K}}, ~~~ \phi_{k^*,\ell} \leq \hat{\phi}_{k^*,\ell}(t) + \Confidence \left(n_\ell(t), \frac{\delta}{3K}\right) \Bigg\}    
\label{assump:PseudowellBehaved1}
\end{equation} 
\noindent

Additionally, define $\mathcal{E}^+_3$ as the event where empirical mean of arm $k^*$ lies below the upper confidence index of arm $k^*$ (constructed with width $\delta/4$) uniformly for all $t \geq 1$, i.e.,
\begin{equation}
\mathcal{E}^+_3 = \Bigg\{ \forall{t \geq 1}, ~~~ \mu_{k^*} \leq \hat{\mu}_{k^*}(t) + \Confidence \left(n_{k^*}(t), \frac{\delta}{4}\right) \Bigg\}    
\label{assump:optimalwellBehaved1}
\end{equation} 
Furthermore, we define $\mathcal{E}^+$ to be the intersection of the three events, i.e., 
\begin{equation}
\mathcal{E}^+ = \mathcal{E}^+_1 \cap \mathcal{E}^+_2 \cap \mathcal{E}^+_3
\label{eq:defEappendix1}
\end{equation}

\noindent
Due to the nature of anytime confidence intervals (See \Cref{eq:anytime}) and union bound over the set of arms, we have $\Pr(\mathcal{E}^+_1) \geq 1 - \frac{\delta}{3}$, $\Pr(\mathcal{E}^+_2) \geq 1 - \frac{\delta}{6}$ and $\Pr(\mathcal{E}^+_3) \geq 1 - \frac{\delta}{2}$, giving us $\Pr(\mathcal{E}^{+}) \geq 1 - \delta$. We now show that under the event $\mathcal{E}^+$, the C-LUCB++ algorithm cannot stop with an arm $k \neq k^*$. We do that through a proof by contradiction.

Suppose, the algorithm stops with arm $k \neq k^*$, i.e., arm $k$ is the only arm in set $\mathcal{A}_t$ or $\max_{\ell \neq k} \tilde{U}_{\ell,\ell}\left(\tau, \frac{\delta}{4}\right) < L_k\left(\frac{\delta}{4K}\right)$. In such a scenario, $\exists \tau, k \neq k^*: \tilde{U}_{k^*}\left(\tau, \frac{\delta}{3K}\right) < L_k\left(\tau, \frac{\delta}{3K}\right)$ or $\tilde{U}_{k^*,k^*}\left(\tau, \frac{\delta}{4}\right) < L_k\left(\frac{\delta}{4K}\right)$ This can only occur if one of the following events occur, 
\begin{enumerate}
    \item $\tilde{U}_{k^*,k^*}\left(\tau, \frac{\delta}{4}\right) < L_k\left(\tau, \frac{\delta}{4K}\right) < L_k\left(\tau, \frac{\delta}{4K}\right) $
    \item $\tilde{U}_{k^*,\ell}\left(\tau, \frac{\delta}{3K}\right) < L_k\left(\tau, \frac{\delta}{3K}\right) \quad \ell \neq k^*$
\end{enumerate}

See that under the event $\mathcal{E}^+$, $\tilde{U}_{k^*,\ell}\left(\tau, \frac{\delta}{3K}\right) > \mu_{k^*}$, $\tilde{U}_{k^*,k^*}\left(\tau, \frac{\delta}{4}\right) < \mu_{k^*}$ and $L_k\left(\tau, \frac{\delta}{3K}\right) < \mu_k \quad \forall{\tau,k}$. This shows that under the event $\mathcal{E}^+$, $\tilde{U}_{k^*, k^*}\left(\tau, \frac{\delta}{4}\right) > L_k\left(\tau, \frac{\delta}{3}\right) \quad \forall{k, \tau}$ and $\tilde{U}_{k^*, \ell}\left(\tau, \frac{\delta}{3K}\right) > L_k\left(\tau, \frac{\delta}{3K}\right) \quad \forall{\ell \neq k^*}$ as $\mu_{k^*} > \mu_{k} \quad \forall{k \neq k^*}$. This implies that the algorithm returns the best arm with probability at least $1 - \delta$ as $\Pr(\mathcal{E}^+) \geq 1 - \delta$.
\end{proof}

\end{document}